%% file: sample_paper.tex
\documentclass[twoside]{article}

\usepackage[accepted]{aistats2023}
\usepackage[round]{natbib}
\usepackage{tcolorbox}
\usepackage{colortbl}

\input{math_commands.tex}

\usepackage{hyperref}
\usepackage{url}
\usepackage{algorithm, algpseudocode}
\usepackage{xfrac}
\usepackage{amsthm}
\usepackage{enumitem}
\usepackage{xcolor}
\usepackage{booktabs}
\usepackage{multirow}
\usepackage{caption}
\usepackage{subcaption}
\usepackage{tikz}

\definecolor{myblue2}{HTML}{4682B4}

\newtheorem{defin}{Definition}
\newtheorem{theorem}{Theorem}

\newtheorem{lemma}{Lemma}

\newtheorem{asump}{Assumption}

\newcommand{\mycite}[1]{\textcolor{myblue2}{\textit{\citep{#1}}}}
\newcommand{\brac}[1]{\left[#1\right]}

\begin{document}

%

%

\twocolumn[

\aistatstitle{Strong Lottery Ticket Hypothesis with $\varepsilon$--perturbation}

\aistatsauthor{ Zheyang Xiong$^\star$ \And Fangshuo Liao$^\star$ \And  Anastasios Kyrillidis }

\aistatsaddress{ Rice University \And  Rice University \And Rice University } ]

\begin{abstract}
The strong Lottery Ticket Hypothesis (LTH) \mycite{ramanujan2019What, Zhou2019Zeros} claims the existence of a subnetwork in a sufficiently large, randomly initialized neural network that approximates some target neural network without the need of training.
We extend the theoretical guarantee of the strong LTH literature to a scenario more similar to the original LTH, by generalizing the weight change in the pre-training step to some perturbation around initialization. 
In particular, we focus on the following open questions: \textit{By allowing an $\varepsilon$-scale perturbation on the random initial weights, can we reduce the over-parameterization requirement for the candidate network in the strong LTH? Furthermore, does the weight change by SGD coincide with a good set of such perturbation?}

We answer the first question by first extending the theoretical result on subset sum \cite{LuekerExponentially} to allow perturbation on the candidates. Applying this result to the neural network setting, we show that such $\varepsilon$-perturbation reduces the over-parameterization requirement of the strong LTH. To answer the second question, we show via experiments that the perturbed weight achieved by the projected SGD shows better performance under the strong LTH pruning.
\end{abstract}

\section{Introduction}
Pruning techniques for over-parameterized neural networks have drawn growing attention in recent years \mycite{Han2015Learning, Li2016Pruning, Wen2016Learning, He2017Channel, Zhu2017ToPrune, Blalock2020What, Wang2019Pruning, Lee2018SNIP, Wang2020Picking}.
Amongst them, the \textit{Lottery Ticket Hypothesis} (LTH) \mycite{frankle2018the,Frankle2019Stabilizing} claims the existence of a small (sparse) subnetwork within a large (dense) neural network such that, when trained in isolation, achieves comparable or even better performance than the original dense network. 
Such subnetworks can be identified by pre-training the dense network and pruning it based on the magnitude of the learned weights \mycite{frankle2018the}. 
Currently, to the best of our knowledge, the LTH lacks of any rigorous theoretical guarantees that justify superior performance of the subnetwork, especially under the pretraining-based pruning; yet, it has been proven to be effective in practice.

The \textit{Strong Lottery Ticket Hypothesis} \mycite{ramanujan2019What, Zhou2019Zeros} leverages a different pruning scheme: given a target dense neural network, and a randomly initialized, \textit{sufficiently over-parameterized} candidate network, there exists a subnetwork in the latter that approximates the former arbitrarily well without the need of training. 
While we usually require a significant over-parameterization in the randomly initialized network, the strong LTH enjoys extensive theoretical guarantees \mycite{malach2020Proving, Pensia2020Opimal, Orseau2020Logarithmic}. 
Yet, the same theory hardly applies to the original LTH, as strong LTH assumes that the candidate weights pruned are fixed at initialization. 
\emph{The fact that LTH pruning is based on weights modified by pre-training motivates us to analyze the approximation behavior that emerges beyond the randomness in the candidate weights.}

Further study on the pre-training process of the LTH shows that the winnning lottery ticket emerges in the early stage of training \mycite{You2019Drawing}, when the loss have not converge to a desirably small value. 
This implies that converging to small training loss is not necessarily the intent of the pre-training steo in the LTH procedure; \textit{in other words, achieving small loss does not necessarily explain how and why pre-training helps pruning in LTH.}
Instead, one could hypothesize that the pre-training --based on loss minimization-- \textit{guide} the weight perturbation to a direction that facilitate the pruning process. 

Based on this hypothesis, our work extends the theoretical guarantee of strong LTH to a scenario more similar to the original LTH, by generalizing the weight change achieved in the pre-training step to some perturbation around the initialization. Our central question is as follows:

\begin{minipage}[t]{0.99\linewidth} 
\begin{tcolorbox}[colback=gray!5,colframe=green!40!black] 
\vspace{-0.2cm}
\textit{``By allowing an $\varepsilon$-scale perturbation on the random initial weights, can we reduce the over-parameterization requirement for the candidate network in the strong LTH? Furthermore, does the weight change by SGD coincide with a good set of such perturbation?''} \vspace{-0.2cm}
\end{tcolorbox}
\end{minipage}

To be more specific, let $f_{\boldsymbol{\theta}}$ be a neural network with parameters $\boldsymbol{\theta}\in\R^d$. 
Formally, an $\varepsilon$-perturbation is a mapping $\mathcal{P}:\R^d\rightarrow\R^d$ such that the maximum entrywise perturbation is bounded in absolute value by $\varepsilon$, i.e., $\|\mathcal{P}(\boldsymbol{\theta}) -\boldsymbol{\theta}\|_\infty \leq \varepsilon$.
We point out that this definition of perturbation generalizes to two existing scenarios: when $\varepsilon = 0$ --i.e., we allow no weight perturbation-- the question above reduces to the original strong LTH. 
When $\varepsilon = \infty$ --i.e., we allow arbitrarily large weight perturbation-- the required over-parameterization for the candidate network is at most the size of the target network. 
In this case, we often use gradient based optimizers such as SGD to find such weight perturbation, but often without the need of pruning. 
Yet, both cases cover only one aspect in pruning and perturbation. 

In this paper, we study the inter-dependence of the two aspects above by treating $\varepsilon$ as a variable. 
In particular, we show that a larger perturbation scale $\varepsilon$, which corresponds to a larger amount of training, would alleviate the over-parameterization requirement, while keeping the accuracy of the pruned neural network the same.
Our contributions can be summarized as below: 
\begin{itemize}[leftmargin=*]
\vspace{-0.2cm}
    \item We consider a generalized version of the subset sum problem where each candidate in the summation is allowed perturbation bounded by a fixed scale $\varepsilon$. We extend the analysis of the subset sum \mycite{LuekerExponentially} to our generalized version, and show that when a larger perturbation is allowed, the required size of the candidate set can be reduced. We empirically validate our theoretical result on the perturbed subset sum problem.
    \vspace{-0.2cm}
    \item Applying the theoretical result above to neural networks, we prove that, when an $\varepsilon$-scale perturbation is allowed, the strong LTH on randomly initialized neural network requires less over-parameterization to achieve a specific approximation error. In particular, the over-parameterization decreases as $\varepsilon$ increases.
    \vspace{-0.2cm}
    \item On neural networks, we empirically show that $i)$ the perturbation that alleviates the overparameterization requirement of the strong LTH can be obtained by projected SGD on the initialized weights; and, $ii)$ under fixed overparameterization, neural networks with a larger freedom over the level of perturbation achieves a higher accuracy after pruning. This result establish the connection between the amount of pre-training and the accuracy of the pruned network.
\end{itemize}

\section{Related Works}

\textbf{Lotter Ticket Hypothesis.} The Lottery Ticket Hypothesis is first proposed by \mycite{frankle2018the}. 
Later works investigate this hypothesis under different scenarios and from different perspectives. 
To name a few, \mycite{Yu2019Playing} studies the LTH in the reinforcement learning setting. 
\mycite{Sabatelli2020Transferability} investigates the performance of lottery tickets in the setting of transfer learning. 
\mycite{You2019Drawing} find that the lottery tickets can be observed in the early phase of training. 
Moreover, a line of work \mycite{Lee2018SNIP, Wang2020Picking, Tanaka2020Pruning, Patil2020PHEW} propose pruning techniques without the need of pre-training. 
Noticeably, \mycite{Tanaka2020Pruning} and \mycite{Patil2020PHEW} even require no training data. 
For a comprehensive survey on the LTH, we refer the readers to \mycite{lange2020_lottery_ticket_hypothesis} and \mycite{cunha2022proving}. 

Several works attempt to explain the LTH. 
\mycite{Evci_Ioannou_Keskin_Dauphin_2022} empirically study the behavior of gradient flow in the pruned network. 
\mycite{Zhang2021Why} assumes that the optimal mask is given, and proves that the pruned network achieves faster convergence and better generalization when trained from initialization.
\mycite{Wolfe2021How} provides a theoretical guarantee for a pruning-after-training fashion. 
However, these works differ from ours in the following: 
$i)$ they usually consider neuron pruning on small neural network archictures (e.g., \mycite{Wolfe2021How} focuses on a two-layer MLP with smooth activations), while our work considers weight pruning of a deep ReLU neural network; 
$ii)$ they consider minimizing the loss on a specific dataset, and require an over-parameterization that scales quadratically with the number of samples, while we approximate a target network with a fixed architecture in terms of the function norm, and require an over-parameterization that scales with the width of the target network.

\textbf{Strong Lottery Ticket Hypothesis.} The strong LTH originates from the empirical observation that, by fixing the weights at initialization and learning the mask over the weights, one can identify subnetworks that achieve comparable accuracy to the dense one with learned weights \mycite{Zhou2019Zeros}. 
\mycite{ramanujan2019What} made this idea more concrete by proposing the \texttt{edge-popup} algorithm to efficiently learn the mask. 
\mycite{malach2020Proving} first proved such hypothesis under the assumption that the dense network's size scales polynomially with the target network's width and depth. 
Leveraging the advantage of weight decomposition and theoretical results on the subset sum problem \mycite{LuekerExponentially}, \mycite{Orseau2020Logarithmic} and \mycite{Pensia2020Opimal} improved the over-parameterization requirement to a logarithm factor times the size of the target network. 
Later work explores different variations of the strong LTH: \mycite{cunha2022proving} and \mycite{Burkholz2022Convolutional} show that the strong LTH holds in the case of convolutional neural networks. 
\mycite{burkholz2022on} extends the strong LTH proof to a universial family of functions. Finally, \mycite{chijiwa2021pruning} further reduces the over-parameterization requirement by employing the iterative randomization.

\section{Notations and Setup}
\label{sec:notation}
\textbf{Notation.}
We use standard lower case letters (e.g. $a$) to denote scalars, bold lower-case letters (e.g. $\mathbf{a}$) to denote vectors, and bold upper-case letters (e.g. $\mathbf{A}$) to denote matrices. For a vector $\mathbf{a}$, we use $\norm{\mathbf{a}}_2$ to denote its $\ell_2$ (Euclidean) norm, and $\norm{\mathbf{a}}_{\infty}$ to denote its $\ell_\infty$ norm. For a matrix $\mathbf{A}$, we use $\norm{\mathbf{A}}_{\max} = \max_{ij}\left|A_{ij}\right|$ to denote its $\max$ norm. We use $\mathbb{P}\paren{\cdot}$ to denote the probability of an event, and $\E\left[\cdot\right]$ to denote the expectation of a random variable. Moreover, $\texttt{Unif}\paren{I}$ denotes the uniform distribution on the interval $I$, $\texttt{Geom}\paren{\cdot}$ denotes the geometric distribution, and $\texttt{Bin}\paren{\cdot,\cdot}$ denotes the binomial distribution. Lastly, we use $\sigma\paren{a} = \max\{0, a\}$ to denote the ReLU activation.

\textbf{Setup.}
Similar to \mycite{Pensia2020Opimal}, our focus is to approximate an $L$-layer, ReLU activated target multi-layer perceptron (MLP) $f(\x)$ by pruning a  $2L$-layer, ReLU activated candidate MLP $g(\x)$. 
For some input vector $\x\in\R^{d_0}$, we assume $f(\x) = f^L(\x)$ has a fixed set of parameters $\{\W^{\ell}\}_{\ell=1}^L$, represented by:
\begin{align*}
    f^{\ell}(\x) &= \begin{cases}
    \W^Lf^{L-1}(\x), & \text{if }\ell=L,\\
    \sigma\paren{\W^{\ell} f^{\ell-1}(\x)}, & \text{if }\ell\in[L-1],\\
    \x, & \text{if }\ell=0,
    \end{cases}
\end{align*}
where $\W^{\ell}\in\R^{d_{\ell}\times d_{\ell-1}}$.
Similarly, let $g(\x) = g^{2L}(\x)$ with parameters $\{\U_{\ell}\}_{\ell=1}^{2L}$, represented by:
\begin{align*}
    g^{\ell}(\x) &= \begin{cases}
    \bold{U}^{2L}g^{2L-1}(\x), & \text{if }\ell=2L,\\
    \sigma\paren{\bold{U}^\ell g^{\ell-1}(\x)}, & \text{if }\ell\in[2L-1],\\
    \x, & \text{if }\ell=0,
    \end{cases}
\end{align*}
where $\U^{\ell}\in\R^{\hat{d}_{\ell}\times\hat{d}_{\ell-1}}$. In particular, $g$ is a neural network with twice the depth of $f$.
We consider the pruning and $\varepsilon$-perturbation of $g(\x)$ with a set of masks for the weights $\mathcal{S} = \{\SP^{\ell}\}_{\ell=1}^{2L}$ and perturbation matrices $\mathcal{Y}=\{\Y^\ell\}_{i=\ell}^{L}$, denoted as $g_{\mathcal{S},\mathcal{Y}}(\x) = g_{\mathcal{S},\mathcal{Y}}^{2L}(\x)$:
\begin{align*}
    g_{\mathcal{S},\mathcal{Y}}^{\ell}(\x)\hspace{-0.1cm} = \hspace{-0.1cm}\begin{cases}
    (\SP^{\ell}\odot(\bold{U}^{\ell}+\Y^{\ell}))g_{\mathcal{S},\mathcal{Y}}^{\ell-1}(\x), &\hspace{-0.3cm}\text{if }\ell=2L,\\
    \sigma\paren{ (\SP^{\ell}\odot(\bold{U}^{\ell}+\Y^{\ell}))g_{\mathcal{S},\mathcal{Y}}^{\ell-1}(\x)}, &\hspace{-0.3cm}\text{if }\ell\in[L-1],\\
    \x, &\hspace{-0.3cm}\text{ if }\ell=0.
    \end{cases}
\end{align*}
Intuitively, $g_{\mathcal{S},\mathcal{Y}}$ is constructed such that, in each layer of $g_{\mathcal{S},\mathcal{Y}}$, the weight matrix $\mathbf{U}^{\ell}$ is first applied with a perturbation matrix $\mathbf{Y}^{\ell}$ and then pruned by applying the mask $\mathbf{S}^{\ell}$
Let $\mathcal{F}_{\mathcal{Y}}$ denote the feasible set of the perturbation $\mathcal{Y}$. We make the following assumption:
\begin{asump}
\label{approx_assump}
We assume the following condition for $f, g$ and $\mathcal{F}_{\mathcal{Y}}$:
\begin{enumerate}[label=(\alph*)]
    \item For all $\ell\in\{0\}\cup[L]$, the weight matrix $\mathbf{W}^{\ell}$ of the target neural network $f$ satisfies $\norm{\mathbf{W}^\ell}\leq 1$ and $\norm{\mathbf{W}^{\ell}}_{\max}\leq\frac{1}{2}$.
    \item The initialization of $g$ satisfies $\mathbf{U}^{2\ell}_{ij}\sim\texttt{Unif}[-1,1]$, and $\mathbf{U}^{2\ell-1}_{ij} = 1$ if $i\leq \sfrac{\hat{d}_{2(\ell-1)}}{2}$ and $\mathbf{U}^{2\ell-1}_{ij} = -1$ if $i > \sfrac{\hat{d}_{2(\ell-1)}}{2}$ for all $\ell\in[L]$ and $j\in[\hat{d}_{2\ell-3}]$.
    \item The feasible set of $\mathcal{Y}$ is defined as:
    \begin{align*}
        \hspace{-0.65cm}\mathcal{F}_{\mathcal{Y}} = \cbrace{\mathcal{Y}:\forall\ell\in[L],\norm{\mathbf{Y}^{2\ell-1}}_{\max}\hspace{-0.15cm} = 0 \wedge \norm{\mathbf{Y}^{2\ell}}_{\max}\leq \varepsilon}.
    \end{align*}
\end{enumerate}
\end{asump}
\textbf{Remark 1.} \textit{Assumption \ref{approx_assump}(a) is similar to \mycite{Pensia2020Opimal}. In particular, the also assume that $\norm{\mathbf{W}^{\ell}}\leq 1$ for all $\ell$. We additional assume that $\norm{\mathbf{W}^{\ell}}_{\max}\leq\frac{1}{2}$ to facilitate the application of our subset sum result in Theorem (\ref{theorem:perturbed_rssp_bound}). This condition can be easily satisfied by most initialization methods. Assumption \ref{approx_assump}(b) states the initialization scheme for the candidate MLP $g$. Intuitively, for weights in layers with odd indices, we initialize the top half to $1$ and the lower half to $-1$, and for weights in layers with even indices, we initialize entries uniformly at random from $[-1, 1]$. Although different from \mycite{Pensia2020Opimal}, we show in later section that this initialization does not affect the difficulty of the approximation, and is only for the convenience of analysis. Assumption \ref{approx_assump}(c) defines the feasible set. In our setting, we only allow weights in layers with even indices to be perturbed. We show in later section that this feasible set is sufficient to establish the dependence of the over-parameterization on the perturbation scale $\varepsilon$.}

To represent the functional approximation of $f$ using $g_{\mathcal{S},\mathcal{Y}}$, we focus on the approximation error defined as:
\begin{align}
    \label{eq:nn_approx_err}
    \min_{\mathcal{Y}\in\mathcal{F}_{\mathcal{Y}},\mathcal{S}}\sup_{\x:\|\x\|\leq 1}\norm{f(\x) -  g_{\mathcal{S},\mathcal{Y}}\paren{\x}}.
\end{align}

\section{Subset Sum with $\varepsilon$-Perturbation}\label{sec:subsetsum}
For each layer in the target network, \mycite{Pensia2020Opimal} constructed a two-layer subnetwork with block structure, such that each block approximates a single entry in the weight matrix of the target network. 
In particular, they obtain a logarithmic-scale over-parameterization by formulating the approximation as a subset sum problem \mycite{LuekerExponentially, daCunha2022Revisiting}
Given a candidate set of values $\{x_i\}_{i=1}^n$ of size $n$ and a target value $z$, the solution to the subset sum problem finds the best approximation of $z$ using the sum of a subset of $\{x_i\}_{i=1}^n$. From an optimization perspective, the optimal approximation error $\eta^\star$ is the solution to the following problem:
\begin{align}
    \label{eq:rssp}
    \eta^\star = \min_{\boldsymbol{\delta}\in\{0, 1\}^n}\left|\sum_{i=1}^n\delta_ix_i - z\right|,
\end{align}
where $\delta_i \in \{0, 1\}$ is the indicator variable on whether $x_i$ is selected in the sum to approximate $z$.

From the perspective of strong LTH, we can treat $z$ as the weight entry in the target network that we wish to approximate, and $\{x_i\}_{i=1}^n$ as the weights in the candidate network we will prune. 
Here, $\delta_i = 1$ means that the $i$-th weight is kept, while $\delta_i = 0$ means that the $i$-th weight is pruned. 
\mycite{LuekerExponentially} shows that, with high probability over the randomness of $\x_i\sim\texttt{Unif}\paren{[-1,1]}$, a candidate set with size of the order $n = \Omega\paren{\log\eta^{-1}}$ is enough to guarantee that $\eta^*\leq \eta$ for all $z\in[-\sfrac{1}{2},\sfrac{1}{2}]$.

As an extension to the strong LTH, our setup incorporates an $\varepsilon$-perturbation on the weights of the random neural network. 
This calls for the attention of extending the random subset sum problem in \Eqref{eq:rssp} to a version with $\varepsilon$-perturbation added. 

In particular, we consider the following joint minimization problem:
\begin{align}
    \label{eq:perturbed_rssp}
    \eta^\star = \min_{\boldsymbol{\delta}\in\{0,1\}^n,\mathbf{y}\in[-\varepsilon,\varepsilon]^n}\left|\sum_{i=1}^n\delta_i\left(x_i + y_i\right) - z\right|.
\end{align}

\begin{figure*}[!ht]
    \centering
    \includegraphics[width=0.8\linewidth]{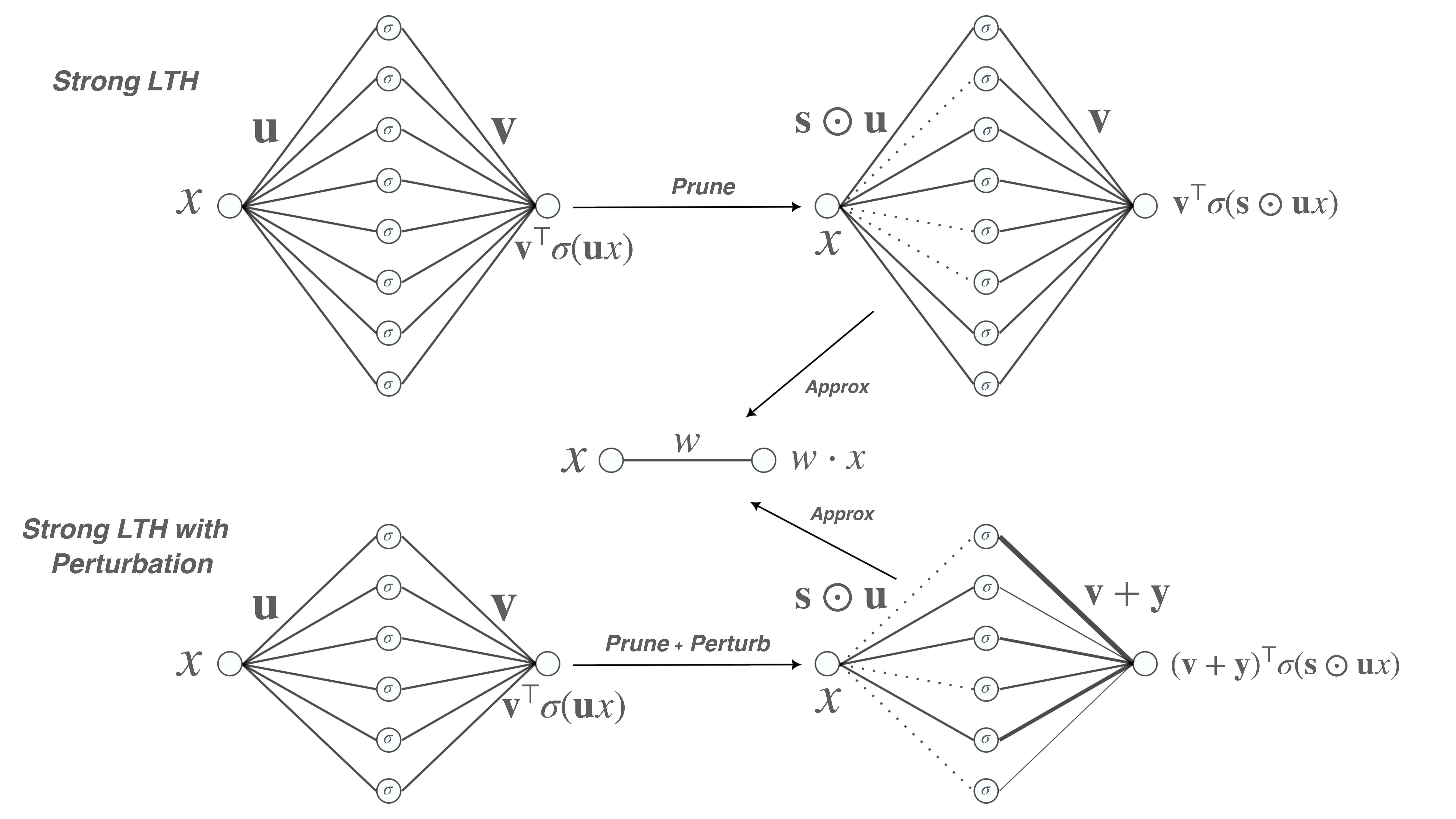}
    \caption{Diagram showing the effect of perturbation on the pruning scheme of strong LTH \mycite{Pensia2020Opimal}.}
    \label{fig:slth_diagram}
\end{figure*}

We denote the values that lead to the optimal approximation error as $\boldsymbol{\delta}^\star$ and $\mathbf{y}^\star$, respectively. 
In words, the above problem aims to select values from the set $\{x_i\}_{i=1}^n$ such that, after potential entrywise perturbation by some tunable $y_i \in [-\varepsilon, \varepsilon]$, the summation of the selected and perturbed $\sum_{i=1}^n\delta_i\left(x_i + y_i\right)$ will approximate $z$.

A central technical difficulty in our work is to extend the result of \mycite{LuekerExponentially} to incorporate such $\varepsilon$-perturbation. 
Intuitively, as the perturbation scale $\varepsilon$ becomes larger, each candidate is susceptible to a larger change in order to better approximate the objective $z$. 
This implies that we should only require a smaller size for the candidate size. This is indeed the case, as we show in the theorem below.

\begin{theorem}
\label{theorem:perturbed_rssp_bound}
Given a candidate set $\{x_i\}_{i=1}^n$ with $x_i\sim\texttt{Unif}\paren{[-1, 1]}$ for all $i\in[n]$. Consider the $\varepsilon$-perturbed random subset sum problem in \Eqref{eq:perturbed_rssp}. Let the number of candidates satisfy $n = K_1 + K_2$ with
\begin{align*}
    K_1=O\paren{\frac{\log\eta^{-1}}{\log\paren{\frac{5}{4} + \frac{\varepsilon}{2}}}}\quad ;K_2 = O\paren{1 + \frac{\log\eta^{-1}}{(1+\varepsilon)}}
\end{align*}
Then with probability at least $1 - \exp\paren{-\frac{K_2(1+\varepsilon)^2}{8(3-\varepsilon)^2}} - \exp\paren{-\frac{K_1}{18}} - \exp\paren{-\max\{\varepsilon,\eta\}K_1}$ we have that all $z\in[-\sfrac{1}{2},\sfrac{1}{2}]$ has a $2\eta$-approximation.
\end{theorem}

\textit{Sketch of proof:} The proof of Theorem (\ref{theorem:perturbed_rssp_bound}) is provided in Appendix (1). Compared with the proof of \mycite{LuekerExponentially}, we included the $\varepsilon$-perturbation when constructing the recurrence of the size of the target range that can be approximated. After introducing this $\varepsilon$-perturbation, we cannot directly apply the techniques in \mycite{LuekerExponentially}. We sketch the proof below while omitting details and the handling of edge cases of different $\varepsilon$:
\begin{enumerate}[leftmargin=*]
    \item We start by defining an indicator function $f_{k,\eta}(z)$ corresponding to the event that $z$ has an $\eta$-approximation by the first $k$ candidates. We show that this indicator function can be recursively defined: $f_{k+1,\eta}(z)$ can be written as a function of $f_{k,\eta}(z)$ and $f_{k,\eta+\varepsilon}(z)$. However, this sequence is hard to control as it involves the $f_{k,\eta+\varepsilon}(z)$. We further study the behavior of the set where $f_{k,\eta+\varepsilon}(z)=1$, and, by introducing the notion of $\varepsilon$-extensionm, we construct another sequence of indicator functions $\{\hat{f}_k\}_{k=1}^n$ that lower bound $f_{k,\eta}$ yet shows the advantage of large $\varepsilon$.
    \item As in \mycite{LuekerExponentially}, we define $p_k$ to be the fraction of $z$ on the interval $[-\sfrac{1}{2},\sfrac{1}{2}]$ such that $\hat{f}_k = 1$. Differently, we show that the expectation of $p_{k+1} - p_k$ is lower bounded by $\sfrac{1}{2}(1-p_k)(p_k+\varepsilon)$. This demonstrates the expected growth $p_{k+1}$ enjoys from $p_k$. Noticeably, this growth is larger when $\varepsilon$ is larger. Note that this property implies a lower bound on the expectation of $p_n$. Next, we apply several techniques to remove the expectation.
    \item We first show the lower bound on $K_1$ such that $p_{K_1} \geq \sfrac{1}{4}$ with a high probability. We do this by partitioning the interval $[0,\sfrac{1}{4}]$ into sub-intervals that represents geometric grown. We then show that the sum of the number of steps of growth that escapes these intervals can be represented as a binomial random variable, which can be bounded by applying Hoeffding's inequality.
    \item Next, starting from $p_k \geq \sfrac{1}{4}$, we lower-bound the summation of $Z_{k+1} = \frac{p_{k+1} -p_k}{p_k(1-p_k)}$ using Azuma's inequality. To relate the summation of $Z_{k+1}$ to the growth of $p_k$, we define a function $\psi(p)$ such that $\psi(p_{k+1}) - \psi(p_k)\geq Z_{k+1}$. In this way, we arrive at a lower bound on $\psi(p_{K_1 + K_2}) - \psi(\sfrac{1}{4})$. Enforcing a lower bound on $p_{K_1 + K_2}$ gives a lower bound on $K_2$.
\end{enumerate}

\textbf{Remark 2.}
\textit{Notice that the lower bound on the candidate set $n$ depends on $K_1$ and $K_2$, where $K_1$ scales inversely with $\log(\sfrac{5}{4}+\sfrac{\varepsilon}{2})$ and $K_2$ scales inversely with $1 + 6\varepsilon$. This implies that $n$ decreases monotonically as $\varepsilon$ increases. We will utilize this result to analyze the approximation error defined in \Eqref{eq:nn_approx_err}.}

\section{Strong Lottery Ticket Hypothesis with $\varepsilon$ Perturbation}
\label{sec:lth}


The theoretical result above provides the ``skeleton'' of techniques on how to incorporate perturbation into the approximation process for the subset sum problem. 
In what follows, we demonstrate how to apply this ``skeleton'' into the approximation of a target neural network. As in \mycite{Pensia2020Opimal}, we start with approximating a single weight entry using a two-layer ReLU neural network.

We highlight the differences between our setting and the setting in \mycite{Pensia2020Opimal} in Figure \ref{fig:slth_diagram}. 
Let the target weight be denoted by $w$. 
For each input value $x$, passing through this weight entry gives an output $w\cdot x$. 

The scheme considered by \mycite{Pensia2020Opimal} is represented in the upper half of the figure, denoted by \textit{``Strong LTH"}.  
In particular, they start with a random two-layer ReLU activated neural network, and performs pruning by applying mask $\mathbf{s}$ to the first layer\footnote[1]{Applying mask to the second layer weights achieves the same effect.} weight $\mathbf{u}$ and arrives at $\mathbf{s}\odot\mathbf{u}$. By choosing an optimal mask $\mathbf{s}$, for any input $x$, they consider the output of the pruned network $\mathbf{v}^\top\sigma\paren{\mathbf{s}\odot\mathbf{u}x}$ as an approximation of $w\cdot x$. 

Our scheme is described in the lower half of the figure, denoted by \textit{``Strong LTH with Perturbation"}. Different from \textit{``Strong LTH"}, we also apply changes to the second layer weight $\mathbf{v}$ by adding a perturbation vector $\mathbf{y}\in [-\varepsilon,\varepsilon]^n$ to it while performing the pruning. With an optimal mask $\mathbf{s}$ and perturbation vector $\mathbf{y}$, we consider $\paren{\mathbf{v} + \mathbf{y}}^\top\sigma\paren{\mathbf{s}\odot\mathbf{u}x}$ as an approximation of $w\cdot x$. As a result of this difference, our scheme potentially requires a smaller over-parameterization. In the lemma below, we show that the over-parameterization in our setting enjoys a monotonic decrease, as we increase $\varepsilon$.
\begin{lemma}
\label{lem:single_weight}
Let $g:\RR\rightarrow\RR$ be a randomly initialized network of the form $g(x)=\bold{v}^\top\sigma(\bold{u}x)$, where $\bold{v},\bold{u}\in\RR^{2n}$, $n=K_1+K_2$ with \begin{align*}
    K_1 &\geq C_1\left(\frac{\log(\eta^{-1})}{\log(\frac{5}{4}+\frac{\varepsilon}{2})}\right);K_2 \geq C_2\left(\frac{\log(\eta^{-1}))}{1+\varepsilon}\right),
\end{align*}
where $u_i=1$ for $i\leq n$, $u_i=-1$ for $i\geq n+1$, and $v_i's$ are drawn from $\texttt{Unif}[-1,1]$. Then there exist $\bold{s}\in\{0,1\}^{2n}, \bold{y} \in[-\varepsilon, +\varepsilon]^{2n}$ such that
\begin{align*}
    \sup_{x:|x|\leq 1}\left|wx-(\bold{v}+ \bold{y})^\top\sigma((\bold{u}\odot \bold{s})x)\right|<\eta,
    \vspace{-0.7cm}
\end{align*}
with probability at least $1 -\delta$ for all $w\in[-\frac{1}{2},\frac{1}{2}]$, with
\begin{align*}
\small
    \delta= \exp\paren{-\tfrac{K_2(1+\varepsilon)^2}{8(3-\varepsilon)^2}} + \exp\paren{-\tfrac{K_1}{18}}+\\
    \exp\paren{-\max\{\varepsilon,\eta\}K_1}
\end{align*}
\end{lemma}
\begin{proof}[Sketch of the Proof]
We defer the detailed proof to the appendix and sketch the proof here.
Similar to \mycite{Pensia2020Opimal}, we decompose $wx$ into $wx = \sigma(wx)-\sigma(-wx)$. By our construction, the first half of the entries in $\mathbf{u}$ are $1$ and the second half are $-1$. Therefore, we decompose $\bold{u},\bold{v},\bold{y},\bold{s}$ by
\begin{align*}
    \bold{u} = \begin{pmatrix}
                \bold{u}_1\\
                \bold{u}_2
               \end{pmatrix},
    \bold{v} = \begin{pmatrix}
                \bold{v}_1\\
                \bold{v}_2
               \end{pmatrix},
    \bold{s} = \begin{pmatrix}
                \bold{s}_1\\
                \bold{s}_2
               \end{pmatrix},
    \bold{y} = \begin{pmatrix}
                \bold{y}_1\\
                \bold{y}_2
               \end{pmatrix},
\end{align*}
where $\uu_1=\bold{1}_n,\uu_2=\bold{-1}_n,\vv_1,\vv_2\in\RR^{n}$, $\s_1,\s_2\in\{0,1\}^n$, and $\y_1,\y_2\in[-\varepsilon,\varepsilon]^n$. This gives:
\begin{align*}
    (\bold{v}+ \bold{y})^\top\sigma((\bold{u}\odot \bold{s})x) & = \underbrace{(\bold{v}_1+ \bold{y}_1)^\top\sigma((\bold{u}_1\odot \bold{s}_1)x)}_{\tau_1} + \\
    &\underbrace{(\bold{v}_2+ \bold{y}_2)^\top\sigma((\bold{u}_2\odot \bold{s}_2)x)}_{\tau_2}
\end{align*}
We then apply Theorem (\ref{theorem:perturbed_rssp_bound}) twice: we first approximate $\sigma(wx)$ using $\tau_1$ and then approximate $-\sigma(-wx)$ using $\tau_2$. Lastly, we combine the two approximations and apply a union bound to complete the proof.
\end{proof}

\textbf{Remark 3.} \textit{Notice that the required number of hidden neurons (over-parameterization) in Lemma (\ref{lem:single_weight}) depends on $K_1$ and $K_2$, where $K_1$ scales inversely with $\log\paren{\sfrac{5}{4} + \sfrac{\varepsilon}{2}}$ and $K_2$ scales inversely with $1+\varepsilon$. Therefore, both terms decreases monotonically as $\varepsilon$ increases. In particular, when $\varepsilon = O(\eta^{-1})$, both $K_1$ and $K_2$ reduce to a constant term. This implies that, as long as the perturbation scale $\varepsilon$ is large enough, a constant over-parameterization is able to approximate any target weight arbitrarily well.}

\textbf{Remark 4.} \textit{Another difference between Lemma (\ref{lem:single_weight}) and Lemma 1 in \mycite{Pensia2020Opimal} is that the latter assumes both $\mathbf{u}$ and $\mathbf{v}$ to be randomly initialized, while we initialize $\mathbf{u}$ deterministically with half $1$s and half $-1$s, and keep $\mathbf{v}$ to be randomly initialized. Our initialization is only for the convenience of applying Theorem (\ref{theorem:perturbed_rssp_bound}). Our initialization is agnostic of the target weight $w$, so it does not alleviate the difficulty of the approximation. Neither does it raise the difficulty, since we can observe that, by setting $\varepsilon$ to zero, the requirement on the over-parameterization reduces to $n = O\paren{\log \eta^{-1}}$. This is the same as Lemma 1 in \mycite{Pensia2020Opimal}.}

Lemma \ref{lem:single_weight} highlights the idea of approximating a single weight entry by pruning and perturbing a randomly initialized two-layer ReLU neural network. Next, we extend this idea to the approximation of a deep neural network. The idea is to approximate each layer in the target network using a two-layer ReLU MLP . Each entry in the weight matrix of the target network is approximated by a subnetwork of the MLP as in Lemma (\ref{lem:single_weight}). Therefore, a concatenation of these MLPs gives an approximation of the target network.

\begin{theorem}
\label{thm:pslth}
Consider approximating $f$ with $g$ as defined above. Assume that assumption (\ref{approx_assump}) holds. Also, assume that for $1\leq \ell\leq L$,
\begin{align*}
    K_1 &= C_1d_{\ell-1}\left(\frac{\log\left(\frac{d_{\ell-1}d_{\ell}L}{\eta}\right)}{\log\paren{\frac{5}{4}+\frac{\varepsilon}{2}}}\right); \\
    K_2 & = C_2d_{\ell-1}\left(\frac{\log\left(\frac{d_{\ell-1}d_{\ell}L}{\eta}\right)}{1+\varepsilon}\right)\\
    & \texttt{dim}(\bold{U}^{2\ell}) = d_\ell \times \left(K_1+ K_2\right),\\
    & \texttt{dim}(\bold{U}^{2\ell-1})= \left(K_1+K_2\right) \times d_{\ell-1}.
\end{align*}
Then with probability at least $1 - 2d_1d_2L\delta$, 
\begin{align*}
    \min_{\mathcal{S},\mathcal{Y}}\sup_{\x:\|\x\|_\infty\leq 1}\|f(\x)-g_{\mathcal{S},\mathcal{Y}}\paren{\x}\|<\eta,
\end{align*}
where $g_{\mathcal{S},\mathcal{Y}}$ is a pruning \& $\varepsilon$-perturbation of $g$, and
\begin{align*}
    \delta= \exp\paren{-\tfrac{K_2(1+\varepsilon)^2}{8(3-\varepsilon)^2}} + \exp\paren{-\tfrac{K_1}{18}}+\\
    \exp\paren{-\max\{\varepsilon,\eta\}K_1}
\end{align*}
\end{theorem}

\begin{proof}[Sketch of Proof]
We defer the detailed proof to the appendix and sketch the proof here. The idea is the same as \mycite{Pensia2020Opimal}: we approximate every entry in the weight matrix in each layer of $f$ by applying lemma (\ref{lem:single_weight}) up to some error. We then notice that, given the assumption that $\norm{\mathbf{W}^{\ell}}\leq 1$, each layer in $f$ is $1$-Lipschitz. Therefore, the corresponding approximation in $g_{\mathcal{S},\mathcal{Y}}$ is also about $1$-Lipschitz. We use this Lipschitzness to control the error during the forward propagation, and arrives at a bound on the error in equation (\ref{eq:nn_approx_err})
\end{proof}

\textbf{Remark 5.} \textit{Theorem (\ref{thm:pslth}) shows similar behavior with Lemma (\ref{lem:single_weight}). The overparameterization requirement depends on $K_1$ and $K_2$, both decreasing monotonically as $\varepsilon$ increases. As long as $\varepsilon = O\paren{\max_{\ell\in[L]}\log\paren{\sfrac{d_{\ell-1}d\ell L}{\eta}}}$, we only require the size of $g$ to be a constant times the size of $f$ to arrive at an $\eta$-approximation of $f$. Moreover, when $\varepsilon = 0$, the required over-parameterization reduces to the same form as in Theorem 1 of \mycite{Pensia2020Opimal}.}


\section{Experiments}
\label{sec:expr}

\subsection{Approximating Neural Nets with SubsetSum and $\varepsilon$ Perturbation}
We would like to see how the amount of weight perturbation affects the required overparametrization. To explore this relationship, we approximate a two-layer, $500$ hidden node target network $g$ by \mycite{Pensia2020Opimal}. Each weight was approximated using a subset sum of $n$ randomly initialized candidates where each candidate was allowed to perturbed by at most $\varepsilon$. In particular, for some given $(\varepsilon$, $\eta)$ and seed $s$, we say $n\in \mathbb{N}$ satisfies the overparametrization requirement if:
\begin{align*}
    \forall w\in g, &~\exists \boldsymbol{\delta}\in\{0,1\}^n, \boldsymbol{y}\in[-\varepsilon,\varepsilon]^n \text{ such that: }\\ & \left|w-\sum_{i=1}^n\delta_i(x_i+y_i)\right|\leq \eta,\\
    & x_i\sim\texttt{Unif}\paren{[l, u]}\forall i=1,...,n,
\end{align*}
where $l$ and $u$ are the bounds of weights $w$ in the target network, $x_i$'s are generated randomly using seed $s$, and every set of $x_i$'s are unique to $w$. 

We randomly generate $10$ sets of $x_i$'s and record the minimum $n$ such that $8$ of such sets leads to the required approximation error.
We vary $\eta$ from $10^{-2}$ to $10^{-4}$ and choose $\varepsilon$ such that $\varepsilon/\eta$ varies between $0$ and $10$. Intuitively, $\varepsilon/\eta$ gives the relative effectiveness of $\varepsilon$. With a fixed $\varepsilon$, a smaller $\eta$ results in a larger $\varepsilon/\eta$, and in the meantime makes the approximation easier. 
We are interested how such $n$ changes as $\eta$ and $\varepsilon/\eta$ change. As $\varepsilon/\eta$ increase, we should expect a smaller size of the candidate set $n$. 

This is indeed the case,  since from \Figref{fig:expr1}, we can observe that for fixed $\eta$, $n$ decreases as $\varepsilon/\eta$ increases. More specifically, as $\varepsilon$ increases, more changes in $\varepsilon$ are required to make a decrease in the minimum over-parameterization requirement $n$, which coincides with Theorem \ref{theorem:perturbed_rssp_bound}.
\begin{figure}[t!]
    \label{fig:expr1}
    \centering
    \includegraphics[width=8cm]{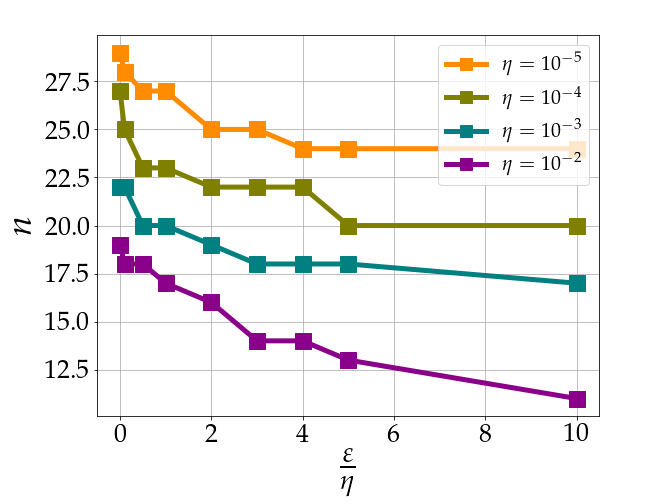}
    \caption{Change of the required size of the candidate set ($n$) v.s. relative perturbation scale (\sfrac{$\varepsilon$}{$\eta$}). 
    }
    \label{fig:expr1}
\end{figure}

\subsection{Perturbation Using Projected Gradient Descent}
We would like to explore the relationship between the weight perturbation achieved by (stochastic) gradient descent ((S)GD) and the desirable weight perturbation for reducing the over-parameterization in the strong LTH approximation. In particular, we hypothesize that
\begin{center}
    \textit{Weight perturbation using gradient-based methods constitutes a good solution compared to the desirable perturbation $\boldsymbol{\delta}^*$ in Equation (\ref{eq:perturbed_rssp}})
\end{center}
We propose a two-stage algorithm to validate this hypothesis. With a given perturbation scale $\varepsilon$, we start with training an over-parameterized neural network using projected gradient descent (PGD) to convergence. Note that, by applying PGD, we guarantee that each value of the neural network weight stays in an $\varepsilon$-neighborhood of initialization. 
This is reflected in lines 3-5 in Algorithm \ref{PGD+StrongLTH}: per iteration, we first complete a regular (S)GD step (line 3), but we then truncate the updated value so that every updated vaule lines in the interval $[-\varepsilon, \varepsilon]$ (line 4). We then apply the resulting update in line 5.

We run \texttt{edge-popup} \mycite{ramanujan2019What} for a range of pruning (sparsity) levels (line 9 in Algorithm \ref{PGD+StrongLTH}). 
This step is interpreted as applying a standard pruning technique --initially designed to start from random initialization, like \texttt{edge-popup}-- but now on the perturbed initialization based on PGD.
We consider the best accuracy amongst all pruning levels (percentage of weights pruned) to be the optimal approximation. We refer the readers to Algorithm \ref{PGD+StrongLTH} for more detail. Here, $\min\{\cdot\}$ refers to the entrywise minimum and $\text{abs}(\cdot)$ refers to the entrywise absolute value. Note that, by applying the projection operation in line 4, we guarantee that $\norm{\mathbf{W}_t - \mathbf{W}_0}_{\max}\leq \varepsilon$ for all $t\in[T]$.

\begin{algorithm}
\caption{PGD+StrongLTH}
\label{PGD+StrongLTH}
\textbf{Input: }Perturbation scale $\varepsilon$, neural network loss $\mathcal{L}$, initial weight $\mathbf{W}_0$, learning rate $\{\alpha_t\}_{t=0}^{T-1}$
\begin{algorithmic}[1]
\label{alg:pgd}
\State $\Delta\mathbf{W}\leftarrow 0$
\For{$t\in\{0,\dots, T-1\}$}
\State $\hat{\mathbf{W}} \leftarrow \Delta\mathbf{W} - \alpha_t\nabla\mathcal{L}(\mathbf{W}_t)$
\State \begin{small}$\Delta\mathbf{W} \leftarrow \texttt{sign}(\hat{\mathbf{W}})\cdot\min\{\texttt{abs}(\hat{\mathbf{W}}),\varepsilon\}$\end{small}
\State $\mathbf{W}_{t+1}\leftarrow \mathbf{W}_0 + \Delta\mathbf{W}$
\EndFor
\State $\ell^* \leftarrow\infty\;, \mathcal{M}^*\leftarrow \text{None}$
\For{pruning level $s\in\{0.1,0.2,\dots,0.9\}$}
    \State $\ell, \mathcal{M} \leftarrow\texttt{Edge-Popup}(\mathcal{L},\mathbf{W}_T, s)$
    \If{$\ell \leq \ell^*$}
    \State $\ell^*\leftarrow \ell\;,\mathcal{M}^*\leftarrow\mathcal{M}$
    \EndIf
\EndFor
\State \textbf{return} Optimal loss $\ell^*$, mask $\mathbf{M}^*$ and sparsity level $s$
\end{algorithmic}
\end{algorithm}
\begin{figure*}[]
    \centering
    \includegraphics[width=1\linewidth]{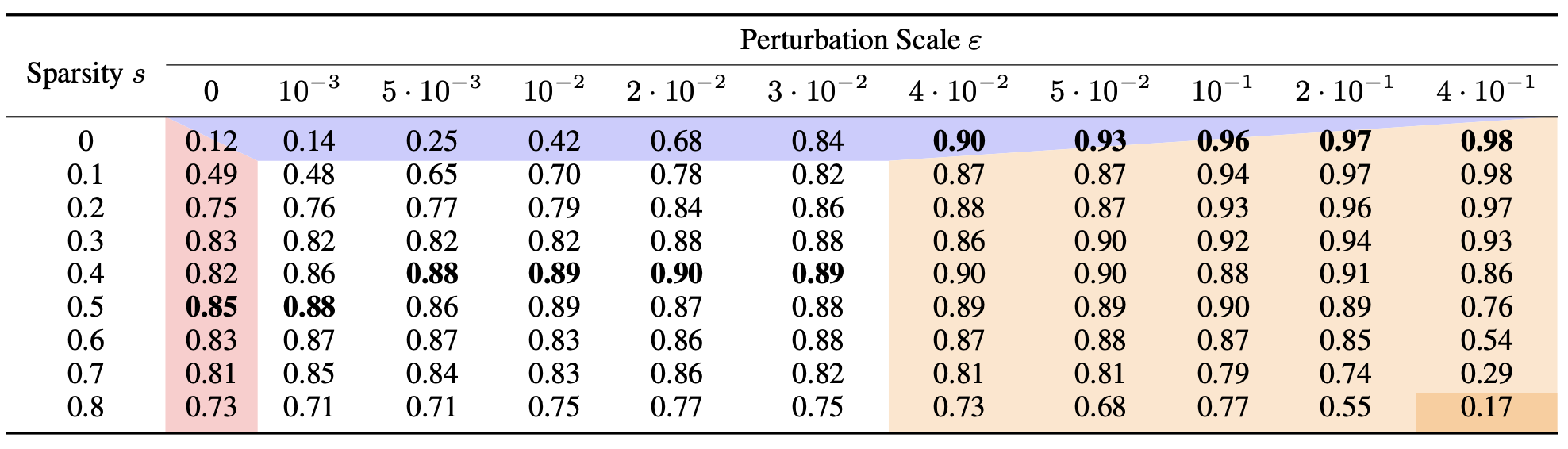}
    \caption{Test accuracy for different pruning level $s$ and perturbation scale $\varepsilon$. For each different perturbation scale (each column), the highest accuracy is marked bold.}
    \label{tb:pgd_result}
\end{figure*}
We train a four-layer multi-layer perceptron (MLP) on the MNIST dataset, with each layer having $500$ hidden nodes. We use Algorithm \ref{PGD+StrongLTH} 
to train the network: we use a learning rate of $0.03$ for PGD and train the network for $100$ epochs; for pruning we use \texttt{edge-popup} with a learning rate of $0.1$ and train the network for $50$ epochs with cosine annealing. 
All weights in the network are initialized from $\texttt{Unif}\paren{[-\sfrac{1}{2},\sfrac{1}{2}]}$, and $\varepsilon$ ranges from $0$ to $0.4$. The results are shown in Table \ref{tb:pgd_result} and Figure \ref{fig:expr2}.

\begin{figure}[!ht]
    \centering
    \begin{subfigure}[b]{0.51\linewidth}
        \centering
        \includegraphics[width=\linewidth]{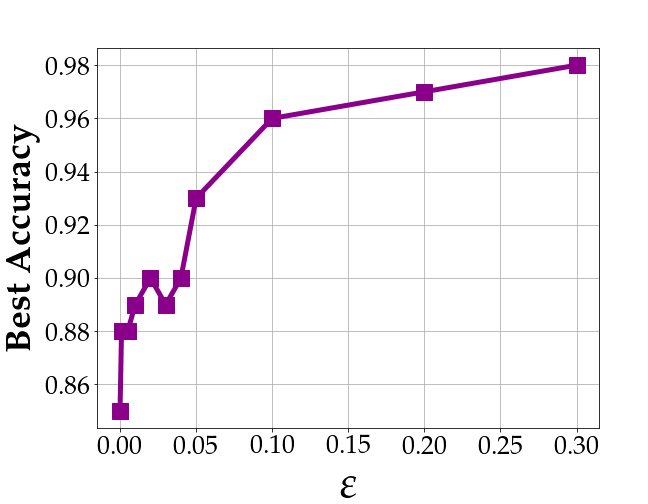}
        \caption{}
        \label{fig:accuracy_eps}
    \end{subfigure}
    \hspace{-0.4cm}
    \begin{subfigure}[b]{0.51\linewidth}
        \centering
        \includegraphics[width=\linewidth]{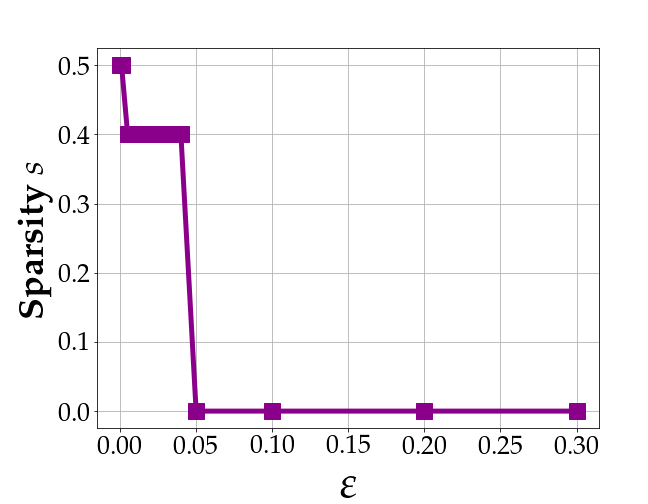}
        \caption{}
        \label{fig:sparsity_eps}
    \end{subfigure}\vspace{-0.2cm}
    \caption{Relationship between perturbation scale $\varepsilon$, optimal sparsity $s^\star$ and the final accuracy. \textbf{\textit{(a)}}: best pruned accuracy across different pruning level on the pruned network versus the perturbation scale. \textbf{\textit{(b)}}: optimal sparsity in pruning versus the perturbation scale.}
    \label{fig:expr2}
    \vspace{-0.3cm}
\end{figure}

Figure \ref{fig:expr2} plots the relationship between the best accuracy among all sparsity (pruning) levels and the perturbation scale \Figref{fig:accuracy_eps}, as well as the sparsity (pruning) level that achieves the best accuracy versus the perturbation scale \Figref{fig:sparsity_eps}. First, in \Figref{fig:accuracy_eps}, one could see that as the perturbation scale increases--where SGD has a larger freedom to find optimize the parameters--the pruned accuracy increases. This observation corroborates our hypothesis that SGD finds a good weight perturbation that facilitate the pruning process. \Figref{fig:sparsity_eps} provides more insight into this behavior: as the perturbation scale increases, the sparsity level that achieves the best accuracy decreases, meaning that the best pruning will prune a smaller number of weights. This is because as the perturbation scales increases, weights learned by SGD are introduced with a dependence on each other. Such dependence becomes stronger as the perturbation scale increases, and will be interrupted by pruning a large number of weights.

Figure \ref{tb:pgd_result} gives a more detailed description of the results. In particular, the best accuracy across all sparsity (pruning) levels for each perturbation scale $\varepsilon$ is \textbf{marked bold}. 
In addition, we use different colors to mark different behaviors our experiment demonstrate when varying $s$ and $\varepsilon$. 
To be more specific, the region \textcolor{red!50}{\textbf{marked red}} represents the behavior studied by the original strong LTH \mycite{Pensia2020Opimal}, where no perturbation is allowed ($\varepsilon = 0$). 
In this case, the optimal accuracy is achieved at a non-trival pruning level ($\sim$0.5). 
The region \textcolor{blue!50}{\textbf{marked purple}} represents standard training of neural network with different extent of training. As $\varepsilon$ increases (more training allowed), the accuracy also increases (higher $\varepsilon$ values allowed).

Finally, the \textcolor{orange!50}{\textbf{orange region}} marks the scenario where the behavior is dominated by SGD. In this region, the large perturbation scale allows SGD to establish a strong dependence between the weight entries, such that any pruning will break this dependence and result in a decrease in the accuracy. In particular, the lowest accuracy (\textcolor{orange!100}{\textbf{marked by dark orange}}) appears when we allow largest extend of training, while enforcing the largest pruning level.

\section{Concluding Remarks}
Our work serves as a further step into understanding how the pre-training process affects the accuracy of neural network pruning. We generalized the weight change in the training process to some $\varepsilon$-perturbation around the initialization, and provide theoretical guarantee to a more general version of the strong Lottery Ticket Hypothesis \mycite{ramanujan2019What, Pensia2020Opimal, malach2020Proving} by introducing the freedom of $\varepsilon$-perturbation into the pruning process. We establish the relationship between the perturbation scale $\varepsilon$ and the over-parameterization requirement of the candidate network. As an intermediate step of our analysis, we also provide the theoretical guarantee of a generalized version of the subset sum problem \mycite{LuekerExponentially}.

Moreover, we explored through experiments whether SGD could find a good perturbation. By testing the a combination of the $\varepsilon$-bounded projected gradient descent and the \texttt{edge-popup} algorithm, we give a positive answer to the question above. We also observed that, as the perturbation scale increases, the optimal accuracy of the pruned network appears at a lower pruning level.

As a next step of our work, one could be interested in exploring the theoretical guarantee for SGD's property of finding a good perturbation. An interesting starting point is the approximation of a weight vector $\mathbf{w}$: given a set of input data points $\{\mathbf{x}_i\}_{i=1}^m$, whether solving the optimization problem of $\min_{\mathbf{U}}\sum_{i=1}^m\norm{\bm{1}^\top\mathbf{U}\mathbf{x}_i - \mathbf{w}^\top\mathbf{x}_i}_2^2$ using gradient descent
\begin{align*}
    \mathbf{U}_{t+1} = \mathbf{U}_{t} -\alpha\frac{\partial}{\partial\mathbf{U}}\sum_{i=1}^m\norm{\bm{1}^\top\mathbf{U}_t\mathbf{x}_i - \mathbf{w}^\top\mathbf{x}_i}_2^2
\end{align*}
will satisfy the descending property
\begin{align*}
    \norm{\mathbf{w} - \paren{\mathbf{U}_{t+1}\odot\mathbf{S}_{t+1}}^\top\bm{1}}_2 < \norm{\mathbf{w} - \paren{\mathbf{U}_{t}\odot\mathbf{S}_{t}}^\top\bm{1}}_2,
\end{align*}
where $\mathbf{S}_t$ is the optimal mask in iteration $t$
\begin{align*}
    \mathbf{S}_t = \text{argmin}_{\mathbf{S}}\norm{\mathbf{w} - \paren{\mathbf{U}_{t}\odot\mathbf{S}}^\top\bm{1}}_2.
\end{align*}
A theoretical analysis on this problem will derive a further connection between the pre-training and the pruning process, and will potentially be an important step in understanding the mystery of the Lottery Ticket Hypothesis.

\clearpage
\bibliography{mybib}

\onecolumn

\appendix

\section{Proof of Theorem 1}
\label{proof:perturbed_rssp_bound}

The subset sum problem considers finding $\boldsymbol{s}\in\{0, 1\}^n$ that minimizes $\ell(z, \boldsymbol{s}) = |z - \sum_{i=1}^ns_ix_i|$ for a given $z$ and given $x_i$'s. Previous work finds that, with $n = \Omega(\log \sfrac{1}{\eta})$, it holds with high probability that there exists $\boldsymbol{s}\in\{0, 1\}^n$ such that $\ell(z, \boldsymbol{s})\leq \eta$. Alternatively, this problem can be started as finding the smallest $n$ such that $\eta^* \leq \eta$ with
\begin{align*}
    \eta^* = \min_{\boldsymbol{s}\in\{0, 1\}^n}\ell(z, \boldsymbol{s})
\end{align*}
In our case, we would like to give the freedom of each $x_i$ to be perturbed for a small degree $\varepsilon$. In particular, we extend the definition of $\ell$ to
\begin{align*}
    \ell(z, \boldsymbol{s},\mathbf{y}) = \left|z - \sum_{i=1}^ns_i(x_i + y_i)\right|
\end{align*}
and seeks condition of $n$ such that $\eta^*\leq \eta$ with
\begin{align}
    \eta^* = \min_{\boldsymbol{s}\in\{0, 1\}^n, \mathbf{y}\in[-\varepsilon, \varepsilon]^n}\ell(z, \boldsymbol{s},\mathbf{y})
\end{align}
If this condition is met for a fixed $z$, we say that such $z$ has an $\eta$ approximation. 
\begin{asump}
\label{base_asump}
Let the candidate values $x_i\sim\texttt{Unif}\paren{[-1, 1]}$ for all $i\in[n]$, and the target value $z\in[-\sfrac{1}{2},\sfrac{1}{2}]$. Let $0 \leq \varepsilon\leq \eta\leq 1$ be given.
\end{asump}
Notice that if $\varepsilon > 1$, then by Hoeffding's inequality,
\begin{align*}
    \mathbb{P}\paren{\left|\sum_{i=1}^nx_i\right|\geq \frac{n\varepsilon}{2}} \leq \exp\paren{-\frac{n\varepsilon^2}{2}}
\end{align*}
When $\left|\sum_{i=1}^nx_i\right|\leq \frac{n\varepsilon}{2}$ holds, we have that $\left|\sum_{i=1}^n(x_i + y_i)\right|$ can be anything in $[-\sfrac{n\varepsilon}{2},\sfrac{n\varepsilon}{2}]$ by varing $y_i$. Therefore, as long as $n = \frac{1}{\varepsilon^2}\log s^{-1}$, it holds with probability at least $1-s$ that $\eta^* = 0$ for all $z\in[-\sfrac{1}{2},\sfrac{1}{2}]$. Thus, our focus is on the case of $\varepsilon\leq 1$.
Under this assumption, we attempts to prove the following theorem
\begin{theorem}
\label{theorem:perturbed_rssp_bound}
Let the number of candidates satisfy $n = K_1 + K_2$ with
\begin{align*}
    K_1=O\paren{\frac{\log\eta^{-1}}{\log\paren{\frac{5}{4} + \frac{\varepsilon}{2}}}}\quad ;K_2 = O\paren{1 + \frac{\log\eta^{-1}}{(1+\varepsilon)}}
\end{align*}
Then with probability at least $1 - \exp\paren{-\frac{K_2(1+\varepsilon)^2}{8(3-\varepsilon)^2}} - \exp\paren{-\frac{K_1}{18}} - \exp\paren{-\max\{\varepsilon,\eta\}K_1}$ we have that all $z\in[-\sfrac{1}{2},\sfrac{1}{2}]$ has a $2\eta$-approximation.
\end{theorem}
We define the indicator function for the existence of $\hat{\eta}$-approximation within the first $k$ candidate.
\begin{align*}
    f_{k,\hat{\eta}}(z) = \indy{\exists s\in\{0,1\}^k, y\in[-\varepsilon,\varepsilon]^k\text{ s.t. }\left|\sum_{i=1}^ks_i(x_i + y_i) - z\right|\leq \hat{\eta}}
\end{align*}
This indicator function has the following recurrence
\begin{align*}
    f_{0,\eta} = \indy{|z|\leq \eta};\quad f_{k+1,\eta} = f_{k,\eta}(z) + \paren{1 - f_{k,\eta}(z)}f_{k,\eta+\varepsilon}(z- x_{k+1})
\end{align*}
Define the following random variable (depending on $\{x_k\}_{i=1}^k$)
\begin{align*}
    p_k = \int_{-\sfrac{1}{2}}^{\sfrac{1}{2}}f_{k,\eta}(z)dz
\end{align*}
This random variable denotes the portion of $z\in[-\sfrac{1}{2},\sfrac{1}{2}]$ that can be approximated within $\eta$ error.
\begin{defin}
For a candidate set $\{x_i\}_{i=1}^n$, and some $k\in\{0\}\cup[n]$, define its $(k, \eta)$-feasible set as
\begin{align*}
    \mathcal{F}_{k,\eta} = \cbrace{z\in\brac{-\sfrac{1}{2},\sfrac{1}{2}}:\exists s\in\{0,1\}^k\text{ s.t. }\left|\sum_{i=1}^ks_ix_i - z\right|\leq \hat{\eta}}
\end{align*}
\end{defin}
By definition, $\mathcal{F}_{k,\eta}$ is the union of finitely many mutually disjoint closed intervals on $\brac{-\sfrac{1}{2},\sfrac{1}{2}}$. Let $\mu$ denote the Lebesgue measure on $\R$. Consider the following definition of $\varepsilon$-extension of a set
\begin{defin}
Let $I\subset[-\sfrac{1}{2},\sfrac{1}{2}]$ be a closed interval. A set $S$ is called an $\varepsilon$-extension of I, denoted $S\in\Xi_{\varepsilon}(I)$ if
\begin{enumerate}
    \item $S\subseteq [-\sfrac{1}{2},\sfrac{1}{2}]\setminus I$
    \item for all $s\in S$, we have that $\min_{a\in I}|s - a|\leq \varepsilon$
    \item $\mu(S) = \min\cbrace{\varepsilon, 1 - \mu(I)}$
\end{enumerate}
By definition, for each $I\subset[-\sfrac{1}{2},\sfrac{1}{2}]$, there is at least one $\varepsilon$-extension of $I$, since we can choose
\begin{align*}
    S = [-\sfrac{1}{2},\sfrac{1}{2}]\cup\begin{cases}
    [-\sfrac{1}{2},\inf I)\cup (\sup I,\sup I -\inf I + \varepsilon - \sfrac{1}{2}] & \text{ if }\inf I \leq \varepsilon - \frac{1}{2}\\
    [\inf I - \varepsilon,\inf I) & \text{ otherwise}
    \end{cases}
\end{align*}
Let $\mathcal{F} = \cup_{j=1}^mI_j$ be a finite union of closed intervals. A set $S$ is called an $\varepsilon$-extension of $\mathcal{F}$, denoted by $S\in\Xi_{\varepsilon}(\mathcal{F})$ if
\begin{enumerate}
    \item $S \subseteq \cup_{j=1}^m\cup_{\xi_j\in\Xi_{\varepsilon}(I_j)}\xi_j\setminus \mathcal{F}$
    \item $\mu(S) = \min\{\varepsilon, 1 - \mu(\mathcal{F})\}$
\end{enumerate}
By lemma \ref{eps_extend}, there is at least one $\varepsilon$-extension of $\mathcal{F}$.
\end{defin}
\begin{lemma}
\label{eps_extend}
There is at least one $\varepsilon$-extension for each $\mathcal{F}$ of the form $\mathcal{F} = \cup_{j=1}^mI_j$.
\begin{proof}
Suppose there is no $\varepsilon$-extension of some $\mathcal{F} = \cup_{j=1}^mI_j$. Consider two cases:

\textbf{Case 1}: $\mu\paren{\mathcal{F}}\geq 1-\varepsilon$. Since $\mathcal{F}$ has no $\varepsilon$-extension, there must be a subset $A$ of $[-\sfrac{1}{2},\sfrac{1}{2}]$ with nonzero Lebesgue measure such that every element in $A$ is at least $\varepsilon$ away from $\mathcal{F}$. This is, however, a contradiction, since by $\mu(\mathcal{F})\geq 1 - \varepsilon$, every point in $[-\sfrac{1}{2},\sfrac{1}{2}]$ must be within $\varepsilon$ distance of $\mathcal{F}$.

\textbf{Case 2}: $\mu\paren{\mathcal{F}}\leq 1 - \varepsilon$. Let $S = \cup_{j=1}^m\cup_{\xi_j\in\Xi_{\varepsilon}(I_j)}\xi_j$. Since $\mathcal{F}$ has no $\varepsilon$-extension, we must have $\mu(S) < 1$. This implies that there exist $a \in[-\sfrac{1}{2},\sfrac{1}{2}]$ such that $a\notin S$. Thus $\inf_{a'\in\mathcal{F}}|a - a'|\geq f$. Let such $a'$ be given, then if $a > a'$, $(a', a' + \varepsilon]$ is an $\varepsilon$-extension, and if $a < a'$, $[a'-\varepsilon, a')$ is an $\varepsilon$-extension. This is a contradiction.
\end{proof}
\end{lemma}
Let $\{S_k\}_{k=0}^n$ be given such that $S_k\in\Xi_{\varepsilon}\paren{\mathcal{F}_{k,\eta}}$. Moreover, let $g_{k}(z) = \indy{z\in S_k}$. We define another recurrence of indicator function
\begin{align*}
    \hat{f}_{k+1}(z) = \hat{f}_k(z) + \paren{1 - \hat{f}_k(z)}\paren{\hat{f}_k(z-x_{k+1}) +g_k(z-x_{k+1})};\quad\hat{f}_0(z) = f_{0,\eta}(z)
\end{align*}
\begin{lemma}
The sequence $\cbrace{\hat{f}_k}_{k=0}^n$ satisfies $\hat{f}_k(z)\leq f_{k,\eta}(z)$ for all $z\in\brac{-\sfrac{1}{2},\sfrac{1}{2}}$.
\begin{proof}
We show this by induction. For $k=0$, we have $\hat{f}_0(z) = f_{0,\eta}(z)$ by definition. Assume $\hat{f}_k(z) \leq f_{k,\eta}(z)$, we would like to show $\hat{f}_{k+1}(z) \leq f_{k+1,\eta}(z)$. To do this, we first notice that, by definition of $g_k$
\begin{align*}
    f_{k,\eta + \varepsilon}\paren{z} = f_{k,\eta}\paren{z} + (1 - f_{k,\eta}\paren{z})\indy{z \in \cup_{\xi_k\in\Xi_{\varepsilon}(\mathcal{F}_{k,\eta})}} \geq f_{k,\eta}\paren{z} + (1 - f_{k,\eta}\paren{z})g_k(z)
\end{align*}
Moreover, if $g_k(z) = 1$, we must have $f_{k,\eta}(z) = 0$. Therefore $(1 - f_{k,\eta}\paren{z})g_k(z) = g_k(z)$. This implies that 
\begin{align*}
    1 \geq f_{k,\eta + \varepsilon}\paren{z} \geq f_{k,\eta}\paren{z} + g_k(z)\geq \hat{f}_k(z) + g_k(z)
\end{align*}
Using this, we have
\begin{align*}
    f_{k+1,\eta}(z) & = f_{k,\eta}(z) + (1 - f_{k,\eta}(z))f_{k,\eta+\varepsilon}(z-x_{k+1})\\
    & \geq f_{k,\eta}(z) + (1 - f_{k,\eta}(z))\paren{\hat{f}_k(z-x_{k+1}) + g_k(z-x_{k+1})}\\
    & = \paren{\hat{f}_k(z-x_{k+1}) + g_k(z-x_{k+1})} + \paren{1 - \hat{f}_k(z-x_{k+1}) - g_k(z-x_{k+1})}f_{k,\eta}(z)\\
    & \geq \paren{\hat{f}_k(z-x_{k+1}) + g_k(z-x_{k+1})} + \paren{1 - \hat{f}_k(z-x_{k+1}) - g_k(z-x_{k+1})}\hat{f}_k(z)\\
    & = \hat{f}_k(z) + \paren{1 - \hat{f}_k(z)}\paren{\hat{f}_k(z-x_{k+1}) + g_k(z-x_{k+1})}\\
    & = \hat{f}_{k+1}(z)
\end{align*}
This completes the proof.
\end{proof}
\end{lemma}
Based on the definition of $\cbrace{\hat{f}_{k}}_{k=0}^n$, we define $\cbrace{\tilde{p}_k}_{k=0}^n$ as
\begin{align*}
    \tilde{p}_k = \int_{-\sfrac{1}{2}}^{\sfrac{1}{2}}\hat{f}_k(z)dz
\end{align*}
Then by definition we have $\tilde{p}_k\leq p_k$, with $\tilde{p}_0 = p_0$. Moreover, we have
\begin{align*}
    \tilde{p}_{k+1} & = \int_{-\sfrac{1}{2}}^{\sfrac{1}{2}}\paren{\hat{f}_k(z) + \paren{1 - \hat{f}_k(z)}\paren{\hat{f}_k(z-x_{k+1}) +g_k(z-x_{k+1})}}dz\\
    & \leq \tilde{p}_k + \int_{-\sfrac{1}{2}}^{\sfrac{1}{2}}\paren{\hat{f}_k(z-x_{k+1}) +g_k(z-x_{k+1})}dz\\
    & \leq \tilde{p}_k + \int_{-\sfrac{1}{2}}^{\sfrac{1}{2}}\paren{\hat{f}_k(u) +g_k(u)}du\\
    & = 2\tilde{p}_k + \mu(S_k)\\
    & \leq 2\tilde{p}_k + \varepsilon
\end{align*}
Furthermore, by definition of $\tilde{p}_{k+1}$, we have $\tilde{p}_{k+1}\leq 1$.
For $\tilde{p}_k$, we can compute its expectation with respect to $\{x_i\}_{i=1}^k$ as
\begin{align*}
    \E\left[\tilde{p}_{k+1}\right] & = \tilde{p}_k + \frac{1}{2}\int_{-1}^1\int_{-\sfrac{1}{2}}^{\sfrac{1}{2}}\paren{1 - \hat{f}_k(z)}\paren{\hat{f}_k(z-x) +g_k(z-x)}dzdx\\
    & = \tilde{p}_k + \frac{1}{2}\int_{-\sfrac{1}{2}}^{\sfrac{1}{2}}\paren{1 - \hat{f}_k(z)}dz\int_{-1}^1\paren{\hat{f}_k(u) +g_k(u)}du\\
    & =\tilde{p}_k + \frac{1}{2}(1 -\tilde{p}_k)(\tilde{p}_k + \mu(S_k))\\
    & = \tilde{p}_k + \frac{1}{2}(1 -\tilde{p}_k)\min\cbrace{1, \tilde{p}_k + \varepsilon}
\end{align*}

\subsection{Growth up to \sfrac{1}{4}}
We define $K_1$ as below
\begin{align*}
    K_1 = \begin{cases}
    \min\left\{k\geq 0: p_k > \frac{1}{4}\right\} & \text{ if }\frac{1}{4} \leq 1- \varepsilon\\
    0 & \text{ otherwise}
    \end{cases}
\end{align*}
To upper bound $K_1$, we consider $\frac{1}{4}\leq 1-\varepsilon$.
\begin{lemma}
\label{growth_prob_bound}
For all $0\leq k\leq K$, it holds that
\begin{align*}
    \mathbb{P}\paren{\tilde{p}_{k+1}\geq \frac{5}{4}\tilde{p}_k +\frac{1}{8}\varepsilon \mid \{x_i\}_{i=1}^k}\geq \frac{1}{6}
\end{align*}
\begin{proof}
Given that $\tilde{p}_k\leq \frac{1}{4}$ , we can show that
\begin{align*}
    \E\left[\tilde{p}_{k+1}\right] =\tilde{p}_k + \frac{1}{2}(1 -\tilde{p}_k)(\tilde{p}_k + \varepsilon) \geq\frac{11}{8}\tilde{p}_k +\frac{3}{8}\varepsilon
\end{align*}
Moreover, recall that we have that $\tilde{p}_{k+1} \leq 2\tilde{p}_k + \varepsilon$. Thus, we can apply the reverse Markov's inequality
\begin{align*}
    \mathbb{P}\paren{p_{k+1}\geq \frac{5}{4}\tilde{p}_k +\frac{1}{8}\varepsilon\mid \{x_i\}_{i=1}^k} 
    & \geq \frac{\E\left[\tilde{p}_{k+1}\right] - \tfrac{5}{4}\tilde{p}_k -\tfrac{1}{8}\varepsilon}{\tilde{p}_{k+1} - \tfrac{5}{4}\tilde{p}_k -\tfrac{1}{8}\varepsilon}\\
    & \geq \frac{\frac{11}{8}\tilde{p}_k +\frac{3}{8}\varepsilon - \tfrac{5}{4}\tilde{p}_k -\tfrac{1}{8}\varepsilon}{2\tilde{p}_k + \varepsilon - \tfrac{5}{4}\tilde{p}_k -\tfrac{1}{8}\varepsilon}\\
    & = \frac{\tfrac{1}{8}\tilde{p}_k + \tfrac{1}{4}\varepsilon}{\tfrac{3}{4}\tilde{p}_k + \frac{7}{8}\varepsilon}\\
    & \geq \frac{1}{6}
\end{align*}
 \end{proof}
\end{lemma}

\begin{lemma}
With probability at least $1 - \exp\paren{-\frac{1}{18}K_1}$ we have that
\begin{align*}
    K \leq O\paren{\frac{\log\eta^{-1}}{\log\paren{\frac{5}{4} + \frac{\varepsilon}{2}}}}
\end{align*}
\begin{proof}
By using $\tilde{p}_k \leq \frac{1}{4}$, it thus follows from lemma (\ref{growth_prob_bound}) that
\begin{align*}
    \mathbb{P}\paren{\tilde{p}_{k+1}\geq \tilde{p}_k\paren{\frac{5}{4} + \frac{\varepsilon}{2}}}\geq \mathbb{P}\paren{\tilde{p}_{k+1}\geq \frac{5}{4}\tilde{p}_k +\frac{1}{8}\varepsilon \mid \{x_i\}_{i=1}^k}\geq \frac{1}{6}
\end{align*}
Denote $\beta = \paren{\frac{5}{4} + \frac{\varepsilon}{2}}$. As in previous work, we define the following partition of the interval $(0,\sfrac{1}{4})$.
\begin{align*}
    I_1 & = (0,\eta]\\
    I_i & = \left(\beta^{i-1}\eta, \beta^i\eta\right]\\
    I_{i^*} & = \left(\beta^{i^*-1}\eta, \frac{1}{4}\right]
\end{align*}
where $i^*$ is the smallest integer such that $\beta^{i^*}\eta \geq \frac{1}{4}$, that is
\begin{align*}
    i^* = \left\lceil\frac{\log\sfrac{1}{4\eta}}{\log\beta}\right\rceil
\end{align*}
For $i > 1$, let $\hat{k}_i$ be given such that $\tilde{p}_{\hat{k}_i}\geq  \beta^{i-1}\eta$, let $\hat{Y}_i$ be the smallest number of steps such that $\tilde{p}_{\hat{k}_i + \hat{Y}_i} > \beta^{i}\eta$. Then we have that $\hat{Y}_i \leq Y_i\sim \texttt{Geom}\paren{\sfrac{1}{6}}$, since, according to lemma (\ref{growth_prob_bound}), we have
\begin{align*}
    \mathbb{P}\paren{\tilde{p}_{k+1}\geq \beta\tilde{p}_k\mid \{x_i\}_{i=1}^k}\geq \frac{1}{6}
\end{align*}
Therefore, for all $K^*$, we have that
\begin{align*}
    \mathbb{P}\paren{K\geq K^*} \leq \mathbb{P}\paren{\sum_{i=1}^{i^*}Y_i\geq K^*} =  \mathbb{P}\paren{B_{K^*} \leq i^*}
\end{align*}
where $B_{K^*}\sim\texttt{Bin}\paren{K^*,\sfrac{4}{7}}$. Given that $\mathbb{E}[B_{K^*}] = \frac{1}{6}K^*$, we can apply the Hoeffding's inequality for binomial distribution
\begin{align*}
    \mathbb{P}\paren{B_{K^*} \leq \frac{1}{6}K^* - t} \leq\exp\paren{-\frac{2t^2}{K^*}}
\end{align*}
choose $K^* = 12i^*$ and $t =\frac{1}{6}K^*$ gives that
\begin{align*}
    \mathbb{P}\paren{K\leq K^*} \geq \mathbb{P}\paren{B_{K^*}\geq i^*} \geq 1 - \exp\paren{-\frac{1}{18}K^{*}}
\end{align*}
\end{proof}
\end{lemma}

\subsection{Growth from $\sfrac{1}{4}$ to $1 - \max\cbrace{\varepsilon,\eta}$}
Recall the recurrence
\begin{align*}
    \E\left[\tilde{p}_{k+1}\right]\geq \tilde{p}_k + \frac{1}{2}(1 -\tilde{p}_k)(\tilde{p}_k + \varepsilon)
\end{align*}
We define
\begin{align*}
    Z_{k+1} = \frac{\tilde{p}_{k+1} - \tilde{p}_k(z)}{(1-\tilde{p}_k)(\tilde{p}_k + \varepsilon)}
\end{align*}
Then we have $\E[Z_{k+1}] \geq \sfrac{1}{2}$. Let $Y_k = -\sfrac{k}{2} + \sum_{i=K_1+1}^{K_1 + k + 1}Z_i$, then $Y_k$ is a submartingale. 
We bound $Z_{k+1}$ as follows
\begin{lemma}
\begin{align*}
    0\leq Z_{k+1}\leq \frac{2}{1+\varepsilon}
\end{align*}
\begin{proof}
We notice that $\tilde{p}_k\leq \tilde{p}_{k+1}\leq \min\left\{2\tilde{p}_k + \varepsilon,1\right\}$. Consider two cases of $p_k$:

\textbf{Case 1}: $\tilde{p}_k \leq \frac{1-\varepsilon}{2}$. In this case, we have $1 - \tilde{p}_k \geq \frac{1 + \varepsilon}{2}$
\begin{align*}
    Z_{k+1} \leq \frac{2\tilde{p}_k + \varepsilon - \tilde{p}_k}{(1-\tilde{p}_k)(\tilde{p}_k + \varepsilon)} = \frac{1}{1 - \tilde{p}_k} \leq \frac{2}{1+\varepsilon}
\end{align*}
\textbf{Case 2}: $\tilde{p}_k \geq \frac{1-\varepsilon}{2}$. In this case, we use $\tilde{p}_{k+1} \leq 1$. Moreover, we have $\tilde{p}_k + \varepsilon \geq \frac{1+\varepsilon}{2}$:
\begin{align*}
    Z_{k+1} \leq \frac{1 - \tilde{p}_k}{(\tilde{p}_k+\varepsilon)(1-\tilde{p}_k)} = \frac{1}{\tilde{p}_k+\varepsilon} \leq \frac{2}{1+\varepsilon}
\end{align*}
\end{proof}
\end{lemma}
Thus,
\begin{align*}
    \left|Y_{k+1} - Y_k\right| = \left|-\frac{1}{2} +Z_{K_1 + k+2}\right| \leq \frac{|3-\varepsilon|}{2+2\varepsilon}
\end{align*}
Let $n = K_1 + K_2 + 1$. Therefore, we can apply Azuma's inequality to get that
\begin{align*}
    \mathbb{P}\paren{\sum_{i=K_1+1}^nZ_i \geq \frac{K_2}{2} - t} & =  \mathbb{P}\paren{-\frac{K_2}{2} + \sum_{i=K_1+1}^nZ_i \geq -t}\\ & = \mathbb{P}\paren{Y_n - Y_0 \geq -t}\\
    &  \geq 1 - \exp\paren{-\frac{2\paren{1+\varepsilon}t^2}{K_2\paren{3-\varepsilon}^2}}
\end{align*}
Let $t =\frac{K_2}{4}$ gives that
\begin{align*}
    \mathbb{P}\paren{\sum_{i=1}^nZ_i \geq\frac{K_2}{4}}\geq 1 - \exp\paren{-\frac{K_2(1+\varepsilon)^2}{8(3-\varepsilon)^2}}
\end{align*}
We use the following function to track the growth of $p_k$, but starting from $\sfrac{1}{4}$.
\begin{align*}
    \psi(p) = \frac{1}{1+\varepsilon}\paren{\log \paren{p+\varepsilon} -\log(1-p)} + \frac{16}{3}p
\end{align*}
\begin{lemma}
For all $p_k\geq \frac{1}{4}$, we have that
\begin{align*}
    \psi(\tilde{p}_{k+1}) \geq \psi(\tilde{p}_k) + Z_{k+1}
\end{align*}
\begin{proof}
We first notice that
\begin{align*}
    \psi(\tilde{p}_{k+1}) - \psi(\tilde{p}_k) = \int_{\tilde{p}_k}^{\tilde{p}_{k+1}}\psi'(p)dp \geq \min_{p\in[\tilde{p}_k,\tilde{p}_{k+1}]}\psi'(p)(\tilde{p}_{k+1}-\tilde{p}_k)
\end{align*}
It suffice to show that
\begin{align*}
    \min_{p\in[\tilde{p}_k,\tilde{p}_{k+1}]} \geq \frac{1}{(\tilde{p}_k+\varepsilon)(1-\tilde{p}_k)}
\end{align*}
The first- and second-order derivative of $\psi$ are
\begin{align*}
    \psi'(p) = \frac{1}{1+\varepsilon}\paren{\frac{1}{p+\varepsilon} + \frac{1}{1-p}} + \frac{16}{3}\\
    \psi''(p) = \frac{1}{1+\varepsilon}\paren{\frac{1}{(1-p)^2} - \frac{1}{\paren{p+\varepsilon}^2}}
\end{align*}
Therefore, $\psi'$ attains its minimum at $p^* = \min\cbrace{1,\frac{1-\varepsilon}{2}}$, and $\psi'$ decreases monotonically on $(\sfrac{1}{4}, p^*]$ and increases monotonically on $[p^*,1]$. Notice that the function $\frac{1}{(\tilde{p}_k+\varepsilon)(1 - \tilde{p}_k)}$ also decreases monotonically on $(\sfrac{1}{4}, p^*]$ and increases monotonically on $[p^*,1]$.
We consider two cases of $p_k$:

\textbf{Case 1}: $\tilde{p}_k\in (\sfrac{1}{4},p^*]$. In this range the function $\frac{1}{(\tilde{p}_k+\varepsilon)(1 - \tilde{p}_k)}$ decreases monotonically. Thus it achieves its maximum at $\tilde{p}_k = \sfrac{1}{4}$ with a value of $\frac{16}{3+12\varepsilon}$. However, since $0\leq p \leq 1$, we have
\begin{align*}
    \min_{p\in[p_k,p_{k+1}]}\psi'(p) \geq \psi'(p^*) \geq \frac{16}{3}
\end{align*}
Thus
\begin{align*}
    \min_{p\in[p_k,p_{k+1}]}\psi'(p) \geq \frac{1}{(\tilde{p}_k+\varepsilon)(1-\tilde{p}_k)}
\end{align*}
\textbf{Case 2}: $\tilde{p}_k\in (p^*, 1]$. Notice that $\psi$ increases monotonically on $(p^*, 1]$. Thus
\begin{align*}
    \min_{p\in[p_k,p_{k+1}]}\psi'(p) = \psi'(p_k)\geq \frac{1}{1+\varepsilon}\paren{\frac{1}{p_k + \varepsilon} + \frac{1}{1 - p_k}} = \frac{1}{\paren{p_k+\varepsilon}(1-p_k)}
\end{align*}
\end{proof}
\end{lemma}
Therefore, we have
\begin{align*}
    \psi(\tilde{p}_{n})  \geq \psi(\tilde{p}_{K_1}) + \sum_{i=K_ 1+1}^nZ_i
\end{align*}
Plugging in the value of $\psi(\tilde{p}_{n})$ and $\psi(\tilde{p}_{K_1})$, and notice that $\psi$ increases monotonically with $p$
\begin{align*}
    -\frac{\log(1-\tilde{p}_{n})}{1+\varepsilon} &  = \psi(\tilde{p}_{n}) - \frac{\log \tilde{p}_{n}}{1+\varepsilon} - \frac{16}{3}\tilde{p}_n\\
    & \geq \psi(\tilde{p}_{K_1}) + \sum_{i=K_ 1+1}^nZ_i - \frac{16}{3}\\
    & \geq \psi\paren{\frac{1}{4}} + \sum_{i=K_ 1+1}^nZ_i - \frac{16}{3}\\
    & = -\frac{1}{1+\varepsilon}\paren{\log 4 +\log\frac{3}{4}} -4 + \sum_{i=K_ 1+1}^nZ_i\\
    & \geq \sum_{i=K_ 1+1}^nZ_i - 4
\end{align*}
Since $\sum_{i=K_ 1+1}^nZ_i \geq \frac{K_2}{4}$ with probability at least $1 - \exp\paren{-\frac{K_2(1+\varepsilon)^2}{8(3-\varepsilon)^2}}$, as long as $K_2 \geq \frac{2\log \eta^{-1}}{1+\varepsilon} + 4 = O(\frac{\log\eta^{-1}}{1+\varepsilon} + 1)$, we have that with high probability $\sum_{i=K_ 1+1}^nZ_i \geq \frac{\log\eta^{-1}}{1+\varepsilon} + 4$, which implies that
\begin{align*}
    -\log(1-\tilde{p}_{n})\geq \log\max\cbrace{\varepsilon,\eta}^{-1} = -\log(\max\cbrace{\eta,\varepsilon})
\end{align*}
which implies that $\tilde{p}_{n} \geq 1- \max\cbrace{\eta,\varepsilon}$. This shows that, with high probability, for $K_1, K_2$ defined above, $n = K_1+ K_2 + 1$ candidates guarantess that each point $z\in[-\sfrac{1}{2},\sfrac{1}{2}]$ either has an $\eta$ approximation or is $\max\{\eta, \varepsilon\}$ away from an $\eta$ approximation. In the case of $\eta \geq \varepsilon$, we have that each $z$ has a $2\eta$ approximation. Otherwise, if $\varepsilon > \eta$, we need an additional set of candidates to grown from $1-\varepsilon$ to $1-\eta$. 

\subsection{Growth from $1-\varepsilon$ to $1-\eta$ under $\varepsilon > \eta$}
Consider another set of candidates $\{\hat{x}_i\}_{i=1}^{K_3}$. By Hoeffding's inequality, we have that
\begin{align*}
    \mathbb{P}\paren{\left|\sum_{i=1}^{K_3}\hat{x}_i\right|\geq (K_3-1)\varepsilon + \eta} \leq \exp\paren{-\frac{\paren{(K_3-1)\varepsilon+\eta}^2}{2K_3}}
\end{align*}
This implies that with probability at least $1 - \exp\paren{-\frac{\paren{(K_3-1)\varepsilon+\eta}^2}{2K_3}}$,
for each $\hat{y}\in[\eta-\varepsilon, \varepsilon-\eta]$, there exists $\mathbf{y}\in[-\varepsilon,\varepsilon]^{K_3}$ such that $\hat{y} = \sum_{i=1}^{K'}\paren{x_i+y_i}$. This implies that with probability at least $1 - \exp\paren{-\frac{\paren{(K_3-1)\varepsilon+\eta}^2}{2K_3}}$, for all $z\in[-\sfrac{1}{2},\sfrac{1}{2}]$ we have that $z$ has an $\eta$ approximation. 
For convenience we can choose $K_3 = K_1 + 1$, which results in a success probability at least $1 - \exp\paren{-\max\{\varepsilon,\eta\}K_1}$
Thus, as long as $n = 2K_1 + K_2+ 1$ with
\begin{align*}
    K_1=O\paren{\frac{\log\eta^{-1}}{\log\paren{\frac{5}{4} + \frac{\varepsilon}{2}}}}\quad ;K_2 = O\paren{1 + \frac{\log\eta^{-1}}{(1+\varepsilon)}}
\end{align*}
we have that with probability at least $1 - \exp\paren{-\frac{K_2(1+\varepsilon)^2}{8(3-\varepsilon)^2}} - \exp\paren{-\frac{K_1}{18}} - \exp\paren{-\max\{\varepsilon,\eta\}K_1}$,
each point in $[-\sfrac{1}{2}, \sfrac{1}{2}]$ either has $\eta$ approximation or lies with $\eta$ distance to a point with $\eta$ approximation. Therefore, each point in $[-\sfrac{1}{2}, \sfrac{1}{2}]$ has an $2\eta$ approximation.

\section{Proof of Theorem 2}

\begin{lemma}
\label{lem:single_weight}
Let $g:\RR\rightarrow\RR$ be a randomly initialized network of the form $g(x)=\bold{v}^\top\sigma(\bold{u}x)$, where $\bold{v},\bold{u}\in\RR^{2n}$, $n=K_1+K_2$, \begin{align*}
    K_1 &\geq C_1\left(\frac{\log((\eta^{-1}))}{\log(\frac{5}{4}+\frac{\varepsilon}{2})}\right), \\ 
    K_2 &\geq C_2\left(\frac{\log((\eta^{-1})))}{1+\varepsilon}\right),
\end{align*}
where $u_i=1$ for $i\leq n$, $u_i=-1$ for $i\geq n+1$, and $v_i's$ are drawn from $\texttt{Unif}[-1,1]$. Then, with probability at least $1 -\delta$, there exist $\bold{s}\in\{0,1\}^{2n}, \bold{y} \in[-\varepsilon, +\varepsilon]^{2n}$ such that
\begin{align*}
    \sup_{x:|x|\leq 1}\left|wx-(\bold{v}+ \bold{y})^\top\sigma((\bold{u}\odot \bold{s})x)\right|<\eta,
    \vspace{-0.5cm}
\end{align*}
for all $w\in[-\frac{1}{2},\frac{1}{2}]$ with
\begin{align*}
\small
    \delta= \exp\paren{-\tfrac{K_2(1+\varepsilon)^2}{8(3-\varepsilon)^2}} + \exp\paren{-\tfrac{K_1}{18}}+\exp\paren{-\max\{\varepsilon,\eta\}K_1}
\end{align*}
\end{lemma}
\begin{proof}
Note that $wx = \sigma(wx)-\sigma(-wx)$ and without loss of generality we assume $w\geq 0$. The case of $w < 0$ can be handled by changing $x$ to $-x$. Furthermore, we decompose $\bold{u},\bold{v},\bold{y},\bold{s}$ by
\begin{align*}
    \bold{u} = \begin{pmatrix}
                \bold{u}_1\\
                \bold{u}_2
               \end{pmatrix},
    \bold{v} = \begin{pmatrix}
                \bold{v}_1\\
                \bold{v}_2
               \end{pmatrix},
    \bold{s} = \begin{pmatrix}
                \bold{s}_1\\
                \bold{s}_2
               \end{pmatrix},
    \bold{y} = \begin{pmatrix}
                \bold{y}_1\\
                \bold{y}_2
               \end{pmatrix},
\end{align*}
where $\uu_1=\bold{1}_n,\uu_2=\bold{-1}_n,\vv_1,\vv_2\in\RR^{n}$, $\s_1,\s_2\in\{0,1\}^n$, and $\y_1,\y_2\in[-\varepsilon,\varepsilon]^n$. Then we have 
\begin{align*}
    (\bold{v}+ \bold{y})^\top\sigma((\bold{u}\odot \bold{s})x) & = (\bold{v}_1+ \bold{y}_1)^\top\sigma((\bold{u}_1\odot \bold{s}_1)x) + (\bold{v}_2+ \bold{y}_2)^\top\sigma((\bold{u}_2\odot \bold{s}_2)x)
\end{align*}
We use the first half of the RHS to approximate $\sigma(wx)$ and use the second half of the RHS to approximate $-\sigma(-wx)$.

\textbf{Approximating $\sigma(wx)$.} Note that since $w\geq 0$, then for $x\leq 0$, $(\bold{v}_1+ \bold{y}_1)^\top\sigma((\bold{u}_1\odot \bold{s}_1)x) = \sigma(wx) = 0$. Consider $x>0$. By definition of $\uu_1$, we have
\begin{align*}
    \vspace{-0.2cm}
    \mathbf{v}_1^\top\sigma\paren{\uu_1x} = \mathbf{v}_1^\top x = \paren{\sum_{i=1}^nv_{1,i}}x.
\end{align*}
Now consider $\paren{\sum_{i=1}^nv_{1,i}}$, Theorem $\ref{theorem:perturbed_rssp_bound}$ states that with probability at least $1-\frac{\delta}{4}$, 
\begin{align*}
    \forall w\in\left[0,\frac{1}{2}\right], & \exists \s_1\in\{0,1\}^n,\y_1\in[-\varepsilon,\varepsilon]^n\text{ s.t. } \left|w-\sum_{i=1}^ns_{1,i}(v_{1,i}+y_{1,i})\right|<\frac{\eta}{2}.
\end{align*}\\
Since 
\begin{align*}
    (\vv_1+\y_1)^\top(\s_1\odot\uu_1)=\sum_{i=1}^ns_{1,i}(v_{1,i}+y_{1,i}),
\end{align*}
with probability at least $1-\frac{\delta}{4}$, we have 
\begin{align*}
    \forall w\in\left[0,\frac{1}{2}\right], & \exists \s_1\in\{0,1\}^n,\y_1\in[-\varepsilon,\varepsilon]^n\text{ s.t. } \left|w-(\vv_1+\y_1)^\top(\s_1\odot\uu_1)\right|<\frac{\eta}{2}.
\end{align*}\\
Since $|x|\leq 1$, with probability at least $1-\frac{\delta}{4}$, we have
\begin{align*}
    \forall w\in\left[0,\frac{1}{2}\right], & \exists \s_1\in\{0,1\}^n,\y_1\in[-\varepsilon,\varepsilon]^n\text{ s.t. } \left|wx-(\vv_1+\y_1)^\top(\s_1\odot\uu_1)x\right|<\frac{\eta}{2}.
\end{align*}\\
Recall that $(\bold{v}_1+ \bold{y}_1)^\top\sigma((\bold{u}_1\odot \bold{s}_1)x) = \sigma(wx) = 0$ for $x\leq 0$. Also, for $x>0$, $\sigma(wx) = wx$ and $(\vv_1+\y_1)^\top\sigma((\uu_1\odot\s_1)x) = (\vv_1+\y_1)^\top(\s_1\odot\uu_1)x$. Therefore, with probability at least $1 -\frac{\delta}{4}$, we have 
\begin{align*}
    \forall w\in\left[0,\frac{1}{2}\right], \exists \s_1\in\{0,1\}^n,\y_1\in[-\varepsilon,\varepsilon]^n\text{ s.t. } 
    \left|\sigma(wx)-(\vv_1+\y_1)^\top\sigma((\uu_1\odot\s_1)x)\right|<\frac{\eta}{2}.
\end{align*}

\textbf{Approximating $-\sigma(-wx)$.} For $x\geq 0$, $(\bold{v}_2+ \bold{y}_2)^\top\sigma((\bold{u}_2\odot \bold{s}_2)x) = -\sigma(-wx) = 0$. Now, consider $x < 0$. By definition of $\uu_2$, it holds that
\begin{align*}
    \vv_2^\top\sigma(\uu_2x) = \paren{\sum_{i=1}^nv_{2,i}}x.
\end{align*}
Therefore, similarly we have that with probability at least $1 - \frac{\delta}{4}$, 
\begin{align*}
    \forall w\in\left[0,\frac{1}{2}\right], \exists \s_2\in\{0,1\}^n,\y_2\in[-\varepsilon,\varepsilon]^n:\quad \left|-\sigma(-wx)-(\vv_2+\y_2)^\top\sigma((\uu_2\odot\s_2)x)\right|<\frac{\eta}{2}.
\end{align*}
Hence by a union bound, with probability at least $1 - \frac{\delta}{2}$,
\begin{align*}
    &\min_{\s,\y}\sup_{x:|x|\leq 1}\left|wx-(\bold{v}+ \bold{y})^\top\sigma((\bold{u}\odot \bold{s})x)\right|\\
    =&\min_{\s_1,\y_1,\s_2,\y_2}\sup_{x:|x|\leq 1}\left|wx-\paren{(\bold{v}_1+ \bold{y}_1)^\top\sigma((\bold{u}_1\odot \bold{s}_1)x)+(\bold{v}_2+ \bold{y}_2)^\top\sigma((\bold{u}_2\odot \bold{s}_2)x)}\right|\\
    =&\min_{\s_1,\y_1,\s_2,\y_2}\sup_{x:|x|\leq 1}\left|\paren{\sigma(wx)-\sigma(-wx)}-\paren{(\bold{v}_1+ \bold{y}_1)^\top\sigma((\bold{u}_1\odot \bold{s}_1)x)+(\bold{v}_2+ \bold{y}_2)^\top\sigma((\bold{u}_2\odot \bold{s}_2)x)}\right|\\
    \leq & \min_{\s_1,\y_1}\sup_{x:|x|\leq 1}\left|\sigma(wx)-(\bold{v}_1+ \bold{y}_1)^\top\sigma((\bold{u}_1\odot \bold{s}_1)x)\right|+\min_{\s_2,\y_2}\sup_{x:|x|\leq 1}\left|-\sigma(-wx)-(\bold{v}_2+ \bold{y}_2)^\top\sigma((\bold{u}_2\odot \bold{s}_2)x)\right|\\
    < & \eta.
\end{align*}
Note that for the case $w\leq 0$, the result has the same probability and the approximation error, so by a union bound, the lemma hold with probability at least $1 -\delta$.
\end{proof}

\begin{lemma}
\label{lem:single_layer}
Let $g:\RR^{d_1}\rightarrow \R^{d_2}$ be a randomly initialized network of the form $g(\bold{x})=\bold{V}\sigma(\bold{U}\bold{x})$, where $\bold{V}\in \RR^{d_2\times 2n}, \bold{U}\in\RR^{2n \times d_1}, n = K_1 + K_2$, 
\begin{align*}
    K_1 &\geq C_1d_1\left(\frac{\log\left(\frac{d_1d_2}{\eta}\right)}{\log\paren{\frac{5}{4}+\frac{\varepsilon}{2}}}\right), \\ 
    K_2 &\geq C_2d_1\left(\frac{\log\left(\frac{d_1d_2}{\eta}\right)}{1+\varepsilon}\right),
\end{align*}
where weights in $\bold{V}$ are drawn i.i.d. from $\texttt{Unif}[-1,1]$, $\mathbf{U} = \begin{pmatrix}
\mathbf{U}^+\\
\mathbf{U}^-
\end{pmatrix}$, with $\mathbf{U}^+$ being a matrix of all $1$s and $\mathbf{U}^-$ being a matrix of all $-1$s . Let $\hat{g}(\bold{x}) = (\bold{S}\odot(\bold{V}+\Y))\sigma((\bold{B}\odot \bold{U})\bold{x})$ be the pruned \ network for masks $\bold{S}\in\{0,1\}^{d_2\times 2n}$, $\bold{B}\in\{0,1\}^{2n\times d_1}$ and perturbation matrix $\bold{Y}\in[-\varepsilon,\varepsilon]^{2n\times d_1}$. Let the target network be $f_{\bold{W}}(\bold{x})=\bold{W}\bold{x}$, then with probability at least $1 - d_1d_2\left(\exp\paren{-\frac{K_2(1+\varepsilon)^2}{8(3-\varepsilon)^2}} - \exp\paren{-\frac{K_1}{18}}-\exp\paren{-\max\{\varepsilon,\eta\}K_1}\right)$, there exist $\bold{S}, \bold{B}, \bold{Y}$ such that
\begin{align*}
    \sup_{x:\|x\|_{\infty}\leq 1}\left\|f_{\bold{W}}(\bold{x})-\hat{g}(\bold{x})\right\|<\eta,
    \vspace{-0.5cm}
\end{align*}
for all $\bold{W}$ such that $\|\bold{W}\|_{\infty}\leq \frac{1}{2}$.
\end{lemma}
\begin{proof}
Since $\mathbf{U}$ can be written as $\begin{pmatrix}
\mathbf{U}^+\\
\mathbf{U}^-
\end{pmatrix}$, with $\mathbf{U}^+$ being a matrix of all $1$s and $\mathbf{U}^-$ being a matrix of all $-1$s, we choose $\hat{\mathbf{B}}$ such that $\hat{\mathbf{B}}\odot\mathbf{U}$ is of the form
\begin{align*}
    \hat{\mathbf{B}}\odot\mathbf{U} = \begin{pmatrix}
    \mathbf{u}_1^+ & 0 & \dots & 0\\
    0 &  \mathbf{u}_2^+ & \dots &  0\\
    \vdots& \vdots & \ddots & \vdots\\
    0 & 0 & \dots & \mathbf{u}_{d_1}^+\\
    \mathbf{u}_1^- & 0 & \dots & 0\\
    0 & \mathbf{u}_2^- & \dots &  0\\
    \vdots& \vdots & \ddots & \vdots\\
    0 & 0 & \dots & \mathbf{u}_{d_1}^-\\
    \end{pmatrix}
\end{align*}
where $\mathbf{u}_j^+ = \bm{1}$ and $\mathbf{u}_j^- = -\bm{1}$. Moreover, we decompose $\mathbf{S}\odot\paren{\mathbf{V} + \mathbf{Y}}$ as
\begin{align*}
    \bold{S} = \begin{pmatrix}
                     \bold{s}_{1,1}^{+\top} & \cdots & \bold{s}_{1,d_1}^{+\top} & \bold{s}_{1,1}^{-\top} & \cdots & \bold{s}_{1,d_1}^{-\top}\\
                     \vdots & & \vdots & \vdots & & \vdots\\
                     \bold{s}_{d_2,1}^{+\top} & \cdots & \bold{s}_{d_2,d_1}^{+\top} & \bold{s}_{d_2,1}^{-\top} & \cdots & \bold{s}_{d_2,d_1}^{-\top}
                     \end{pmatrix}\\
    \bold{V} = \begin{pmatrix}
                     \bold{v}_{1,1}^{+\top} & \cdots & \bold{v}_{1,d_1}^{+\top} & \bold{v}_{1,1}^{-\top} & \cdots & \bold{v}_{1,d_1}^{-\top}\\
                     \vdots & & \vdots & \vdots & & \vdots\\
                     \bold{v}_{d_2,1}^{+\top} & \cdots & \bold{v}_{d_2,d_1}^{+\top} & \bold{v}_{d_2,1}^{-\top} & \cdots & \bold{v}_{d_2,d_1}^{-\top}
                     \end{pmatrix}\\
    \bold{Y} = \begin{pmatrix}
                     \bold{y}_{1,1}^{+\top} & \cdots & \bold{y}_{1,d_1}^{+\top} & \bold{y}_{1,1}^{-\top} & \cdots & \bold{y}_{1,d_1}^{-\top}\\
                     \vdots & & \vdots & \vdots & & \vdots\\
                     \bold{y}_{d_2,1}^{+\top} & \cdots & \bold{y}_{d_2,d_1}^{+\top} & \bold{y}_{d_2,1}^{-\top} & \cdots & \bold{y}_{d_2,d_1}^{-\top}
                     \end{pmatrix}
\end{align*}
where each $\bold{s}_{i,j}^{\pm},\bold{v}_{i,j}^{\pm}, \bold{y}_{i,j}^{\pm}\in\RR^{n/d_1}$. Then we have 
\begin{align*}
    \brac{(\bold{S}\odot(\bold{V}+\Y))\sigma((\hat{\bold{B}}\odot \bold{U})\bold{x})}_{i} & = \sum_{j=1}^{d_1}((\bold{v}_{i,j}^++\y_{i,j}^+)\odot\bold{s}_{i,j}^+)^\top\sigma(\bold{u}_{j}^+x_j) + \\
    &\quad\quad\quad \sum_{j=1}^{d_1}((\bold{v}_{i,j}^-+\y_{i,j}^+)\odot\bold{s}_{i,j}^+)^\top\sigma(\bold{u}_{j}^-x_j)
\end{align*}
Letting $\bold{v}_{ij}= \begin{pmatrix}
\bold{v}_{ij}^+\\
\bold{v}_{ij}^-
\end{pmatrix}, \bold{s}_{ij}= \begin{pmatrix}
\bold{s}_{ij}^+\\
\bold{s}_{ij}^-
\end{pmatrix}, \bold{y}_{ij}= \begin{pmatrix}
\bold{y}_{ij}^+\\
\bold{y}_{ij}^-
\end{pmatrix}$ and $\bold{y}_{ij}= \begin{pmatrix}
\bold{y}_{ij}^+\\
\bold{y}_{ij}^-
\end{pmatrix}$, we then have
\begin{align*}
    \brac{(\bold{S}\odot(\bold{V}+\Y))\sigma((\hat{\bold{B}}\odot \bold{U})\bold{x})}_{i} & = \sum_{j=1}^{d_1}((\bold{v}_{i,j}+\y_{i,j})\odot\bold{s}_{i,j})^\top\sigma(\bold{u}_{j}x_j)
\end{align*}
Now define the event 
\begin{align*}
    F_{i,j,\eta}:=\left\{\sup_{w:|w|\leq \frac{1}{2}}\inf_{\substack{\bold{s}_{i}\in\{0,1\}^{2n/d_1},\\\bold{y}_{i,j}\in[-\varepsilon,\varepsilon]^{2n/d_1}}}\sup_{x:|x|\leq 1}\left|wx-((\bold{v}_{i,j}+\y_{i,j})\odot \bold{s}_{i,j})^\top\sigma(\bold{u}_{j}x)\right|<\eta\right\}.
\end{align*}
Define $F_{\eta}:=\bigcap_{i=1}^{d_2}\bigcap_{j=1}^{d_1}F_{i,j,\eta}$, then 
\begin{align*}
    \PP\left(F_{\frac{\eta}{d_1d_2}}\right) \geq 1 - d_1d_2\left(\exp\paren{-\frac{K_2(1+\varepsilon)^2}{8(3-\varepsilon)^2}} - \exp\paren{-\frac{K_1}{18}}-\exp\paren{-\max\{\varepsilon,\eta\} K_1}\right).
\end{align*}
On event $F_{\frac{\eta}{d_1d_2}}$, we have
\begin{align*}
    &\sup_{\|\bold{W}\|_{\infty}\leq \frac{1}{2}}\inf_{\bold{S},\bold{B},\bold{Y}}\sup_{\|\bold{x}\|_{\infty}\leq 1}\left\|\bold{W}\bold{x}-(\bold{S}\odot(\bold{V}+\Y))\sigma((\bold{B}\odot \bold{U})\bold{x})\right\|\\
    \leq & \sup_{\|\bold{W}\|_{\infty}\leq \frac{1}{2}}\inf_{\bold{S},\bold{Y}}\sup_{\|\bold{x}\|_{\infty}\leq 1}\left\|\bold{W}\bold{x}-(\bold{S}\odot(\bold{V}+\Y))\sigma((\hat{\bold{B}}\odot \bold{U})\bold{x})\right\|\\
    \leq & \sup_{\|\bold{W}\|_{\infty}\leq \frac{1}{2}}\inf_{\bold{S},\bold{Y}}\sup_{\|\bold{x}\|_{\infty}\leq 1}\sum_{i=1}^{d_2}\left|\sum_{j=1}^{d_1}w_{i,j}x_j-\sum_{j=1}^{d_1}((\bold{v}_{i,j}+\y_{i,j})\odot \bold{s}_{i,j})^\top\sigma(\bold{u}_{j}x_j)\right|\\
    \leq & \sup_{\|\bold{W}\|_{\infty}\leq \frac{1}{2}}\inf_{\bold{S},\bold{Y}}\sup_{\|\bold{x}\|_{\infty}\leq 1}\sum_{i=1}^{d_2}\sum_{j=1}^{d_1}\left|w_{i,j}x_j-((\bold{v}_{i,j}+\y_{i,j})\odot \bold{s}_{i,j})^\top\sigma(\bold{u}_{j}x_j)\right|\\
    <&d_1d_2\frac{\eta}{d_1d_2}\\
    = & \eta.
\end{align*}
\end{proof}

With the help of this lemma, we are ready to prove theorem 2. Recall that our goal is to approximate an $L$-layer, ReLU activated target multi-layer perceptron (MLP) $f(\x)$ by pruning a  $2L$-layer, ReLU activated candidate MLP $g(\x)$. 
For some input vector $\x\in\R^{d_0}$, we assume $f(\x) = f^L(\x)$ has a fixed set of parameters $\{\W^{\ell}\}_{\ell=1}^L$, represented by:
\begin{align*}
    f^{\ell}(\x) &= \begin{cases}
    \W^Lf^{L-1}(\x), & \text{if }\ell=L,\\
    \sigma\paren{\W^{\ell} f^{\ell-1}(\x)}, & \text{if }\ell\in[L-1],\\
    \x, & \text{if }\ell=0,
    \end{cases}
\end{align*}
where $\W^{\ell}\in\R^{d_{\ell}\times d_{\ell-1}}$.
Similarly, let $g(\x) = g^{2L}(\x)$ with parameters $\{\U_{\ell}\}_{\ell=1}^{2L}$, represented by:
\begin{align*}
    g^{\ell}(\x) &= \begin{cases}
    \bold{U}^{2L}g^{2L-1}(\x), & \text{if }\ell=2L,\\
    \sigma\paren{\bold{U}^\ell g^{\ell-1}(\x)}, & \text{if }\ell\in[2L-1],\\
    \x, & \text{if }\ell=0,
    \end{cases}
\end{align*}
where $\U^{\ell}\in\R^{\hat{d}_{\ell}\times\hat{d}_{\ell-1}}$. In particular, $g$ is a neural network with twice the depth of $f$.
We consider the pruning and $\varepsilon$-perturbation of $g(\x)$ with a set of masks for the weights $\mathcal{S} = \{\SP^{\ell}\}_{\ell=1}^{2L}$ and perturbation matrices $\mathcal{Y}=\{\Y^i\}_{i=1}^{L}$, denoted as $g_{\mathcal{S},\mathcal{Y}}(\x) = g_{\mathcal{S},\mathcal{Y}}^{2L}(\x)$:
\begin{align*}
    g_{\mathcal{S},\mathcal{Y}}^{\ell}(\x) &= \begin{cases}
    (\SP^{2L}\odot(\bold{U}^{2L}+\Y^{2L}))g_{\mathcal{S},\mathcal{Y}}^{2L-1}(\x), & \text{ if }\ell=2L,\\
    \sigma\paren{ (\SP^{\ell}\odot(\bold{U}^{\ell}+\Y^{\ell}))g_{\mathcal{S},\mathcal{Y}}^{\ell-1}(\x)}, &\text{ if }i\in[L-1],\\
    \x, & \text{ if }\ell=0.
    \end{cases}
\end{align*}
Let $\mathcal{F}_{\mathcal{Y}}$ denote the feasible set of the perturbation $\mathcal{Y}$. Also recall our assumptions for the setup
\begin{asump}
\label{approx_assump}
We assume the following condition for $f, g$ and $\mathcal{F}_{\mathcal{Y}}$:
\begin{enumerate}
    \item For all $\ell\in\{0\}\cup[L]$, the weight matrix $\mathbf{W}^{\ell}$ of the target neural network $f$ satisfies $\|\W^\ell\|\leq 1$ and $\norm{\mathbf{W}^{\ell}}_{\infty}\leq\frac{1}{2}$.
    \item The initialization of $g$ satisfies $\mathbf{U}^{2\ell}_{ij}\sim\texttt{Unif}[-1,1]$, and $\mathbf{U}^{2\ell-1}_{ij} = 1$ if $i\leq \sfrac{\hat{d}_{2(\ell-1)}}{2}$ and $\mathbf{U}^{2\ell-1}_{ij} = -1$ if $i > \sfrac{\hat{d}_{2(\ell-1)}}{2}$ for all $\ell\in[L]$ and $j\in[\hat{d}_{2\ell-3}]$.
    \item The feasible set of $\mathcal{Y}$ is defined as
    \begin{align*}
        \mathcal{F}_{\mathcal{Y}} = \cbrace{\mathcal{Y}:\forall\ell\in[L],\norm{\mathbf{Y}^{2\ell-1}}_{\max} = 0 \text{ and } \norm{\mathbf{Y}^{2\ell}}_{\max}\leq \varepsilon}.
    \end{align*}
\end{enumerate}
\end{asump}

We focus on the approximation error defined as:
\begin{align}
    \label{eq:nn_approx_err}
    \min_{\mathcal{Y}\in\mathcal{F}_{\mathcal{Y}},\mathcal{S}}\sup_{\x:\|\x\|\leq 1}\norm{f(\x) -  g_{\mathcal{S},\mathcal{Y}}\paren{\x}}.
\end{align}

We state the theorem here for convenience.
\begin{theorem}
Consider approximating $f$ with $g$ as defined above. Assume that assumption (\ref{approx_assump}) holds. Also, assume that for $1\leq \ell\leq L$,
\begin{align*}
    K_1 = C_1d_{\ell-1}\left(\frac{\log\left(\frac{d_{\ell-1}d_{\ell}L}{\eta}\right)}{\log\paren{\frac{5}{4}+\frac{\varepsilon}{2}}}\right);& \quad K_2 = C_2d_{\ell-1}\left(\frac{\log\left(\frac{d_{\ell-1}d_{\ell}L}{\eta}\right)}{1+\varepsilon}\right)\\
    \texttt{dim}(\bold{U}^{2\ell}) = d_\ell \times \left(K_1+ K_2\right);& \quad \texttt{dim}(\bold{U}^{2\ell-1}) = \left(K_1+K_2\right) \times d_{\ell-1}.
\end{align*}
Then with probability at least $1 - 2d_1d_2L\left(\exp\paren{-\frac{K_2(1+\varepsilon)^2}{8(3-\varepsilon)^2}} - \exp\paren{-\frac{K_1}{18}}-\exp\paren{-\max\{\varepsilon,\eta\} K_1}\right)$, 
\begin{align*}
    \min_{\mathcal{S},\mathcal{Y}}\sup_{\x:\|\x\|_\infty\leq 1}\|f(\x)-g_{\mathcal{S},\mathcal{Y}}\paren{\x}\|<\eta,
\end{align*}
where $g_{\mathcal{S},\mathcal{Y}}$ is a pruning \& $\varepsilon$-perturbation of $g$.
\end{theorem}
\begin{proof}
By Lemma \ref{lem:single_layer}, for $\ell$-th layer, with probability $1 - d_1d_2\left(\exp\paren{-\frac{K_2(1+\varepsilon)^2}{8(3-\varepsilon)^2}} - \exp\paren{-\frac{K_1}{18}}-\exp\paren{-\max\{\varepsilon,\eta\}K_1}\right)$, we have 
\begin{equation}
    \label{eq:cond1}
    \sup_{\W^{\ell}\in\RR^{d_\ell\times d_{\ell-1}}:\|\W^\ell\|\leq 1, \|\W^\ell\|_\infty\leq\frac{1}{2}}\min_{S^{2\ell},S^{2\ell-1},\Y^{\ell}}\sup_{\x:\|\x\|\leq 1}\|\W^\ell\x-(\SP^{2\ell}\odot\U^{2\ell})\sigma((\SP^{2\ell-1}\odot(\U^{2\ell-1}+\Y^{\ell}))\x)\|<\frac{\eta}{2L}.
\end{equation}
Since ReLU is $1$-Lipschitz, with same probability, we have 
\begin{equation}
    \label{eq:cond2}
    \sup_{\W^{\ell}\in\RR^{d_\ell\times d_{\ell-1}}:\|\W^\ell\|\leq 1, \|\W^\ell\|_\infty\leq\frac{1}{2}}\min_{S^{2\ell},S^{2\ell-1},\Y^{\ell}}\sup_{\x:\|\x\|\leq 1}\|\sigma(\W^\ell\x)-\sigma((\SP^{2\ell}\odot\U^{2\ell})\sigma((\SP^{2\ell-1}\odot(\U^{2\ell-1}+\Y^{\ell}))\x))\|<\frac{\eta}{2L}.
\end{equation}
Then with probability at least $1 - 2d_1d_2L\left(\exp\paren{-\frac{K_2(1+\varepsilon)^2}{8(3-\varepsilon)^2}} - \exp\paren{-\frac{K_1}{18}}-\exp\paren{-\max\{\varepsilon,\eta\}K_1}\right)$, (\ref{eq:cond1}) and (\ref{eq:cond2}) hold simultaneously for every layer $1\leq \ell \leq L$. Equation (\ref{eq:cond2}) implies for $1\leq \ell \leq L-1$,
\begin{align*}
    \norm{\sigma\paren{\mathbf{W}^{\ell+1}g_{\mathcal{S},\mathcal{Y}}^{2\ell}(\x)}- g_{\mathcal{S},\mathcal{Y}}^{2(\ell+1)}(\x)}\leq \frac{\eta}{2L}\norm{g_{\mathcal{S},\mathcal{Y}}^{2\ell}(\x)}
\end{align*}
Since $\norm{\mathbf{W}^{\ell}}\leq 1$ for all $\ell\in[L]$, we have that
\begin{align*}
    \norm{g_{\mathcal{S},\mathcal{Y}}^{2(\ell+1)}(\x)} \leq \frac{\eta}{2L}\norm{g_{\mathcal{S},\mathcal{Y}}^{2\ell}(\x)} + \norm{\sigma\paren{\mathbf{W}^{\ell+1}g_{\mathcal{S},\mathcal{Y}}^{2\ell}(\x)}} \leq \paren{1 + \frac{\eta}{2L}}\norm{g_{\mathcal{S},\mathcal{Y}}^{2\ell}(\x)}
\end{align*}
This implies that, for all $\x$ such that $\norm{\x}\leq 1$,
\begin{align*}
    \norm{g_{\mathcal{S},\mathcal{Y}}^{2\ell}(\x)} \leq \paren{1 + \frac{\eta}{2L}}^{\ell-1}\norm{\x}\leq \paren{1 + \frac{\eta}{2L}}^{\ell-1}
\end{align*}
Thus, we have that for all $\x$ such that $\norm{\x}\leq 1$
\begin{align*}
    \norm{f^{\ell+1}(\x) - g_{\mathcal{S},\mathcal{Y}}^{2(\ell+1)}(\x)} & = \norm{\sigma\paren{\mathbf{W}^{\ell+1}f^{\ell}(\x)} -  g_{\mathcal{S},\mathcal{Y}}^{2(\ell+1)}(\x)}\\
    & \leq \norm{\sigma\paren{\mathbf{W}^{\ell+1}f^{\ell}(\x)} - \sigma\paren{\mathbf{W}^{\ell+1}g_{\mathcal{S},\mathcal{Y}}^{2\ell}(\x)}} + \norm{\sigma\paren{\mathbf{W}^{\ell+1}g_{\mathcal{S},\mathcal{Y}}^{2\ell}(\x)} - g_{\mathcal{S},\mathcal{Y}}^{2(\ell+1)}(\x)}\\
    & \leq \norm{f^{\ell}(\x) - g_{\mathcal{S},\mathcal{Y}}^{2\ell}(\x)} + \frac{\eta}{2L}\norm{g_{\mathcal{S},\mathcal{Y}}^{2\ell}(\x)}\\
    & \leq \norm{f^{\ell}(\x) - g_{\mathcal{S},\mathcal{Y}}^{2\ell}(\x)} + \paren{1+\frac{\eta}{2L}}^{\ell-1}\frac{\eta}{2L}
\end{align*}
Solving the recurrence thus gives
\begin{align*}
    \norm{f(\x) - g_{\mathcal{S},\mathcal{Y}}(\x)} & = \norm{f^{L}(\x) - g_{\mathcal{S},\mathcal{Y}}^{2L}(\x)}\\
    & \leq \sum_{i=1}^{L-1}\paren{1+\frac{\eta}{2L}}^{i-1}\frac{\eta}{2L}\\
    &=\frac{\eta}{2L}\frac{2L}{\eta}\left(\left(1+\frac{\eta}{2L}\right)^L-1\right)\\
    &<e^{\eta/2}-1\\
    &<\eta.
\end{align*}
\end{proof}
\end{document}

%% file: math_commands.tex
\usepackage{amsmath,amsfonts,bm}

\newcommand{\x}{\mathbf{x}}
\newcommand{\W}{\mathbf{W}}

\newcommand{\U}{\mathbf{U}}
\newcommand{\SP}{\mathbf{S}}
\newcommand{\Y}{\mathbf{Y}}

\newcommand{\indy}[1]{\mathbb{I}\left\{#1\right\}}
\newcommand{\uu}{\mathbf{u}}
\newcommand{\vv}{\mathbf{v}}
\newcommand{\y}{\mathbf{y}}
\newcommand{\s}{\mathbf{s}}

\newcommand{\paren}[1]{\left(#1\right)}
\newcommand{\cbrace}[1]{\left\{#1\right\}}
\newcommand{\norm}[1]{\left\|#1\right\|}

\def\RR{{\mathbb R}}

\def\PP{{\mathbb P}}

\def\Figref#1{Figure~\ref{#1}}

\def\eqref#1{equation~\ref{#1}}
\def\Eqref#1{Equation~\ref{#1}}

\def\1{\bm{1}}

\DeclareMathAlphabet{\mathsfit}{\encodingdefault}{\sfdefault}{m}{sl}
\SetMathAlphabet{\mathsfit}{bold}{\encodingdefault}{\sfdefault}{bx}{n}

\newcommand{\E}{\mathbb{E}}

\newcommand{\R}{\mathbb{R}}

